\documentclass[11pt]{article}
\usepackage[T1]{fontenc}
\usepackage{amsmath}%
\usepackage{amsthm}
\usepackage{amsxtra}%
\usepackage{amsfonts}%
\usepackage{amssymb}%
\usepackage[margin=3cm]{geometry}
\usepackage{color}
\usepackage{url}
\usepackage{algorithm}
\usepackage{algorithmic}
\usepackage{graphicx} 
\usepackage{subfigure} 
\newcommand{\R}{\mathbb{R}}
\newcommand{\Z}{\mathbb{Z}}
\newcommand{\N}{\mathbb{N}}
\newcommand{\E}{\mathbb{E}}
\newcommand{\lip}{\,\text{Lip}}
\renewcommand{\P}{\mathbb{P}}
\newcommand{\F}{\mathcal{F}}

\renewcommand{\hat}[1]{\widehat{#1}}
\newcommand{\eps}{\varepsilon}

\newtheorem{theorem}{Theorem}
\newtheorem*{theorem*}{Theorem}
\theoremstyle{plain}

\title{Anomaly detection and classification for streaming data using PDEs\thanks{The second author was supported by NSF grant DMS-1500829.}}
\author{Bilal Abbasi\thanks{Department of Mathematics and Statistics, McGill University. ({\tt bilal.abbasi.ba@gmail.com})} \and Jeff Calder\thanks{School of Mathematics, University of Minnesota. ({\tt jcalder@umn.edu})}\and Adam M. Oberman\thanks{Department of Mathematics and Statistics, McGill University. ({\tt adam.oberman@mcgill.ca})}}

\begin{document}

\maketitle

\begin{abstract}
Nondominated sorting, also called Pareto Depth Analysis (PDA), is widely used in multi-objective optimization and has recently found important applications in multi-criteria anomaly detection. Recently, a partial differential equation (PDE) continuum limit was discovered for nondominated sorting leading to a very fast approximate sorting algorithm called \emph{PDE}-based ranking. We propose in this paper a fast real-time streaming version of the PDA algorithm for anomaly detection that exploits the computational advantages of PDE continuum limits. Furthermore, we derive new PDE continuum limits for sorting points within their nondominated layers and show how the new PDEs can be used to classify anomalies based on which criterion was more significantly violated. We also prove statistical convergence rates for PDE-based ranking, and present the results of numerical experiments with both synthetic and real data.
\end{abstract}

\section{Introduction}
\label{intro}

Sorting, or ordering of multivariate data is an important and challenging problem in many fields of computational science. Since there is no canonical linear ordering for multivariate data, many different notions of ordering have been proposed in the literature~\cite{liu1999multivariate}, and the problem is very much application dependent.

In the context of multiobjective optimization, ordering by dominance relations has achieved prominence. A general multiobjective optimization problem involves finding among a set of feasible solutions those that minimize a collection of objectives. One feasible solution is said to \emph{dominate} another if it gives a smaller value for \emph{every} objective. The collection of feasible solutions that are not dominated by any other solution are called \emph{Pareto-optimal} or \emph{nondominated}. In the database community the Pareto-optimal solutions are called the \emph{skyline} of the dataset~\cite{kossmann2002}. 

The notion of Pareto-optimality is widely used in evolutionary algorithms for multiobjective optimization~\cite{srinivas1994}, such as the Nondominated Sorting Genetic Algorithm (NSGA-II)~\cite{deb2002}, the Strength Pareto Evolutionary Algorithm (SPEA)~\cite{zitzler2001spea2,zitzler1999multiobjective}, and the Pareto envelope-based selection algorithm (PESA)~\cite{corne2000pareto}, among many others (see~\cite{ghosh2004evolutionary} for a survey). Central to many of these algorithms is the assignment of a fitness to each feasible solution based on sorting all the feasible solutions via dominance. 

The NSGA-II algorithm assigns its fitness level via \emph{nondominated sorting}, sometimes called \emph{Pareto Depth Analysis} (PDA), which arranges the feasible solutions into layers by repeatedly peeling off the Pareto-optimal solutions. Nondominated sorting has also found applications in gene selection and ranking~\cite{hero2002}, anomaly detection~\cite{hsiao2012,hsiao2015}, and multiquery image retrieval~\cite{hsiao2015pareto}. As it turns out, nondominated sorting is equivalent to the \emph{longest chain problem}, which has a long history in combinatorics and probability~\cite{bollobas1988,hammersley1972,ulam1961}.
 
Due to the wide use of NSGA-II, there has been significant interest in fast algorithms for nondominated sorting~\cite{deb2002,jensen2003,fortin2013generalizing}. Recently, Calder et al.~\cite{calder2014} established a continuum limit for nondominated sorting that corresponds to solving a Hamilton-Jacobi partial differential equation (HJE). This result shows that there is a simple asymptotic structure underlying nondominated sorting, and this opens the door to extremely fast algorithms based on exploiting this structure. Calder et al.~\cite{calder2015PDE} recently proposed a sublinear algorithm for approximate nondominated sorting called \emph{PDE-based ranking} that is based on estimating the distribution of the data and solving the HJE numerically.

The purpose of this paper is twofold. First, we show how to use PDE-based ranking to significantly improve the performance of algorithms that are based on nondominated sorting. To illustrate this in a concrete setting, we propose a new real-time version of the multi-criteria PDA anomaly detection algorithm from~\cite{hsiao2015} that uses PDE-based ranking in place of nondominated sorting. The computational complexity is reduced by an order of magnitude (from quadratic to linear), and this allows the model to be updated in real-time upon the acquisition of each additional data sample. We also prove in Theorem \ref{thm:rate} a statistical convergence rate for the PDE-based ranking continuum approximation. 

Second, we present a new partial differential equation (PDE) continuum limit for ordering solutions \emph{within} the layers generated by nondominated sorting in two dimensions. This new continuum limit allows us to efficiently explore the tradeoff between multiple objectives. In the context of multi-criteria anomaly detection, we show how to use this PDE continuum limit to classify anomalies based on which criterion is more significantly violated. We give a derivation of these new continuum limits and present a convergence analysis. In both cases, we trade exact algorithms of high computational complexity for fast approximate algorithms that are convergent, meaning that the error in the approximation goes to zero as the sample size grows.

This paper is organized as follows. In Section \ref{sec:prev} we review the continuum limit for nondominated sorting from~\cite{calder2014}, and present the PDA multicriteria anomaly detection algorithm from \cite{hsiao2015}. In Section \ref{sec:new}, we derive two new PDE continuum limits for ordering points within Pareto fronts, in Section \ref{sec:schemes} we construct fast upwind schemes for solving these PDEs numerically. In Section \ref{sec:algorithm}, we present our fast PDE-based anomaly detection and classification algorithm in the context of streaming data, and in Section \ref{sec:numerics} we present the results of numerical experiments for both real and synthetic data streams.

\section{Previous work}
\label{sec:prev}

\subsection{Nondominated sorting}
\label{sec:ndom}

Nondominated sorting arranges a set of points in $\R^d$ into layers by repeatedly peeling off the coordinatewise minimal points. The coordinatewise partial order on $\R^d$ is defined by
\[x \leqq y \ \ \iff \ \ \forall i \ \ x_i \leq y_i \ \ \ \ (x,y \in \R^d).\]
 Let $S_n=\{X_1,\dots,X_n\}\subset \R^d$ be a collection of $n$ points in $\R^d$. We say a point $X_i \in S_n$ is minimal (or nondominated), if no other non-identical point in $S_n$ is smaller with respect to the coordinatewise partial order $\leqq$. The first nondominated layer, denoted $\F_1$, consists of the minimal points from $S_n$. The second nondominated layer, denoted $\F_2$, consists of the minimal points from $S_n\setminus \F_1$ and the $k^{\rm th}$ layer is defined recursively by
\[\F_k = \text{Minimal points from } S_n \setminus (\F_1\cup \cdots \cup \F_{k-1}).\]

Nondominated sorting refers to the process of arranging the set $S_n$ into the nondominated layers $\F_1,\F_2,\F_3,\dots$, which are also called \emph{Pareto fronts}. See Figure \ref{fig:demo} for a demonstration of nondominated sorting applied to random points. The index of the Pareto front that a point $X_i \in S_n$ lies on is often called the \emph{Pareto depth} or \emph{Pareto rank} of $X_i$, and provides the fitness score for the NSGA-II algorithm. In the context of multiobjective optimization, $d$ represents the number of objectives.

\begin{figure}
\centering
\subfigure[$n=50$]{\includegraphics[clip=true,trim=20 20 20 20,height=0.23\textwidth]{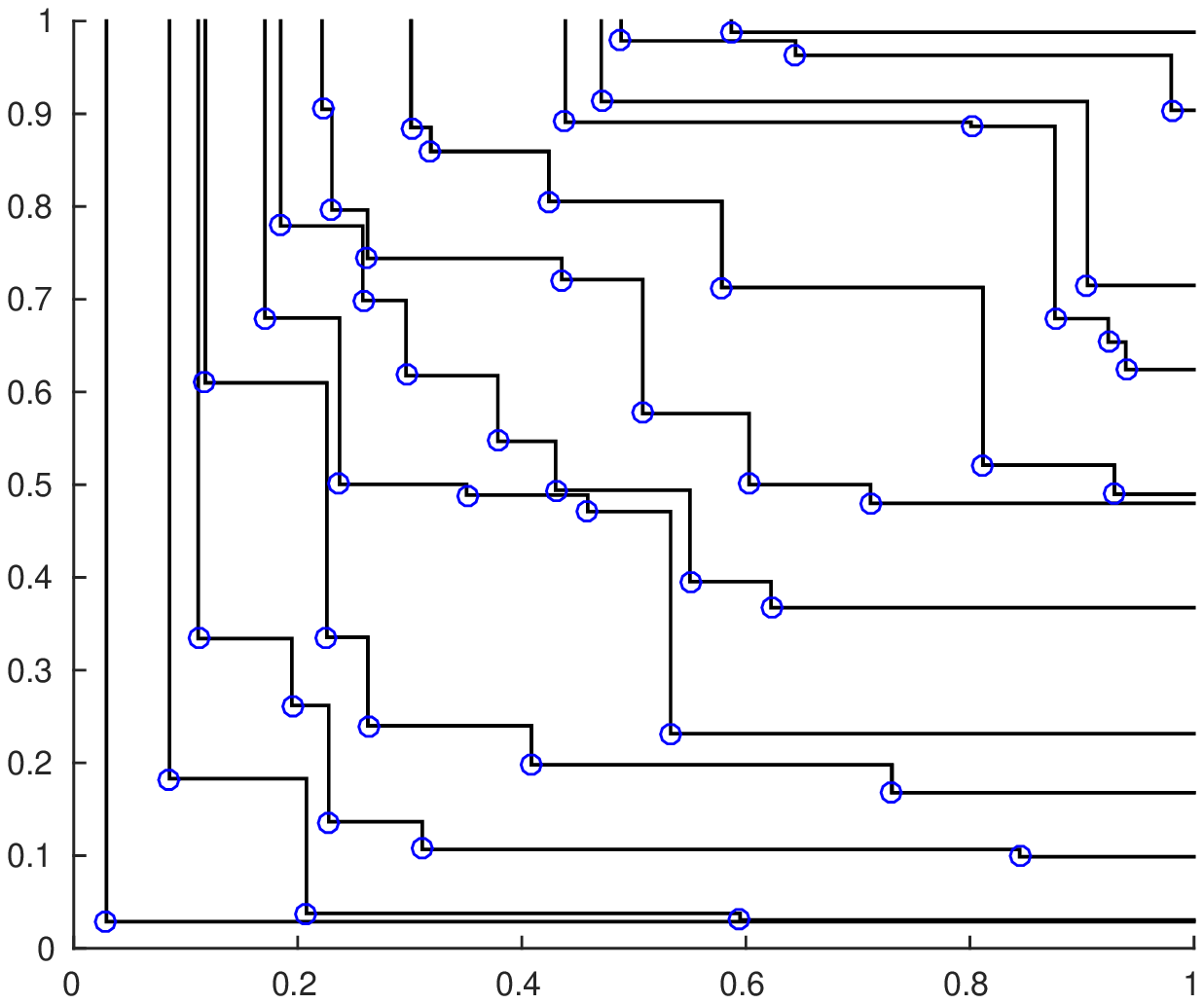}}
\subfigure[$n=10^4$]{\includegraphics[clip=true,trim=20 20 20 20,height=0.23\textwidth]{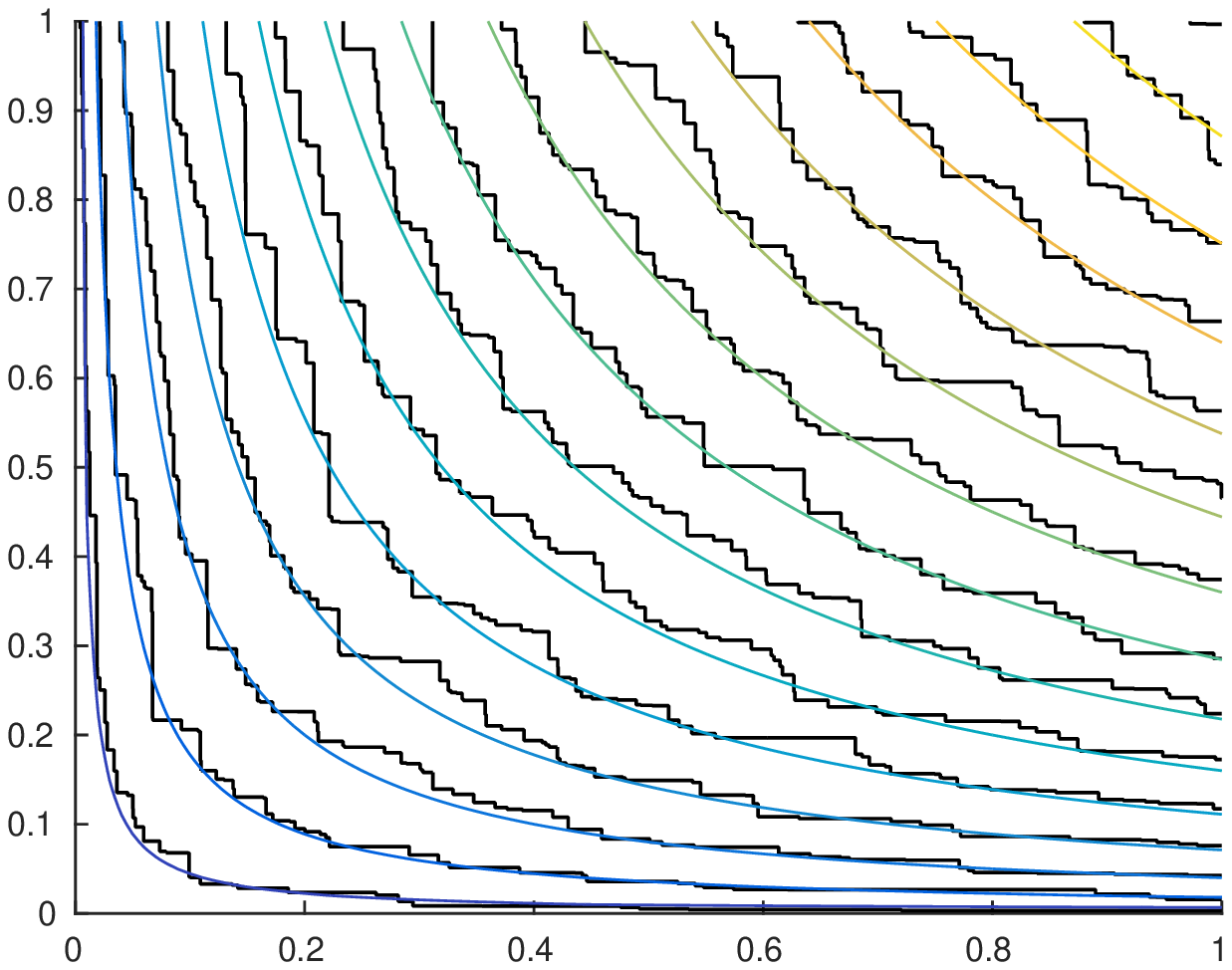}}
\subfigure[$n=10^6$]{\includegraphics[clip=true,trim=20 20 20 20,height=0.23\textwidth]{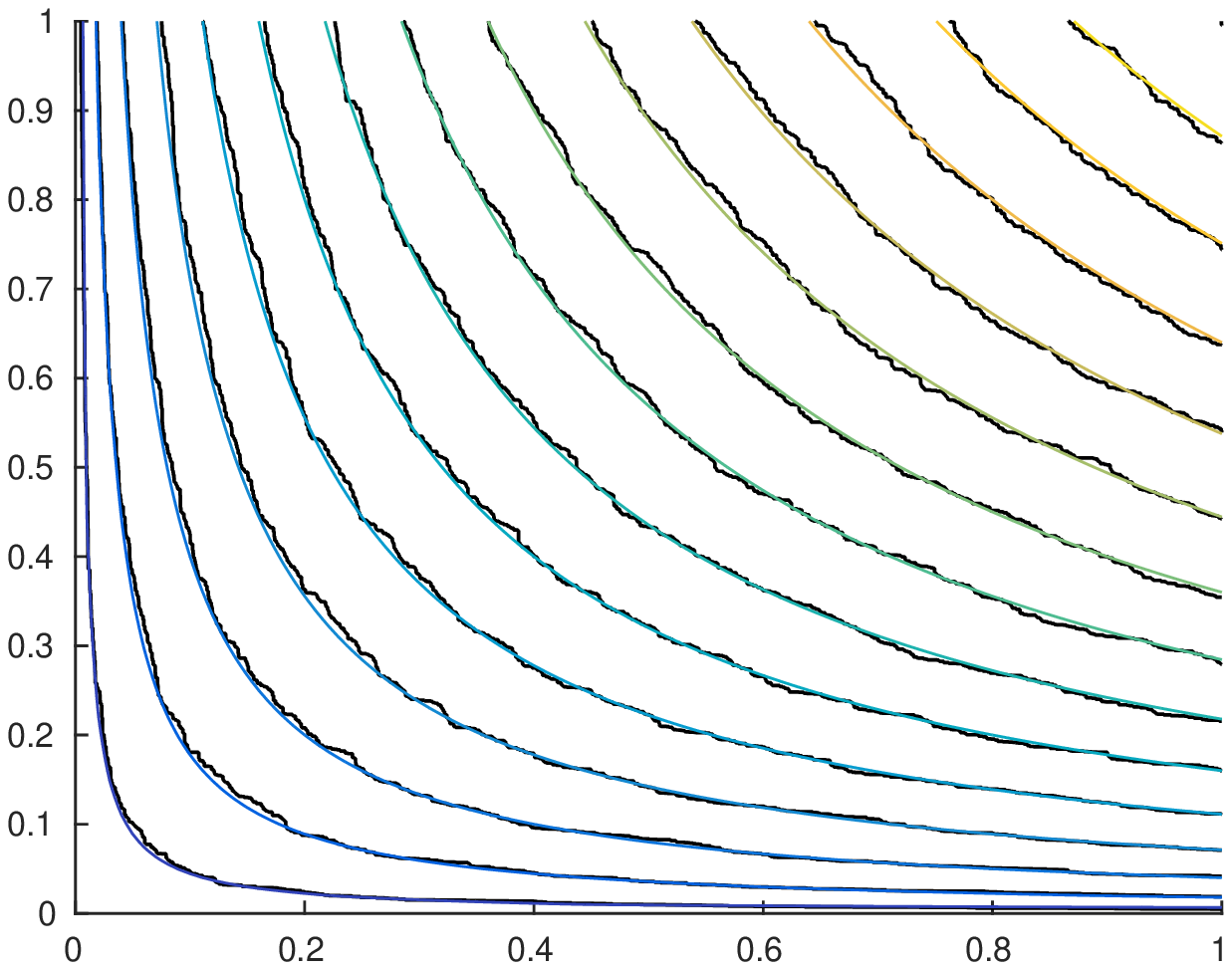}}
\caption{Illustration of nondominated sorting of \emph{i.i.d.}~random variables $X_1,\dots,X_n$ drawn from the uniform distribution on $[0,1]^2$. }
\label{fig:demo}
\end{figure}
The original nondominated sorting algorithm proposed in \cite{deb2002} requires $O(dn^2)$ memory and operations. The quadratic memory complexity in $n$ renders the algorithm intractable for even moderate $n$. Jensen~\cite{jensen2003} proposed an algorithm with asymptotic complexity of $O(n(\log n)^{d-1})$ as $n\to \infty$.  The two dimensional version of Jensen's algorithm was discovered independently in the combinatorics community by Felsner and Wernisch~\cite{felsner1999}. 
Fortin et al.~\cite{fortin2013generalizing} recently made some improvements to Jensen's algorithm regarding its treatment of points with identical coordinates. The exponential complexity of the Jensen-Fortin algorithm with respect to $d$ suggests it may not be useful for high dimensional problems. However, recent numerical results have suggested a better asymptotic complexity as $d\to \infty$ with $n$ fixed~\cite{hsiao2015}. We also mention there are several other notable approaches to nondominated sorting~\cite{shi2005fast,caoMATLAB,fang2008efficient}.

Calder et al.~\cite{calder2014} discovered a Hamilton-Jacobi equation continuum limit for nondominated sorting. The result applies in the setting where $S_n=\{X_1,\dots,X_n\}$ is a sequence of \emph{i.i.d.}~random variables with probability density $f$ on the unit box $(0,1)^d$. Define  the \emph{Pareto depth function} $U_n:S_n \to \N$ associated with nondominated sorting of $S_n$ by  $U_n(X_i) = k$ if and only if $X_i \in \F_k$. The following continuum limit was established in \cite{calder2014,calder2015direct}.
\begin{theorem}[HJE Continuum Limit]\label{thm:main}
With probability one
\[n^{-\frac{1}{d}}U_n \longrightarrow C_d u\ \ \ \text{uniformly on } [0,1]^d \text{ as }  n\to \infty\]
where $C_d>0$ is a constant depending only on $d$, and $u\in C^{0,\frac{1}{d}}([0,1]^d)$ is the unique nondecreasing viscosity solution of the Hamilton-Jacobi equation
\begin{equation}\label{eq:HJE}
\left\{\begin{aligned}
u_{x_1}\cdots u_{x_d} &= f&&\text{in } (0,1]^d\\
u&=0&&\text{on } \partial (0,1)^d\setminus (0,1]^d.
\end{aligned}\right.
\end{equation}
\end{theorem}
Here $u_{x_i} := \frac{\partial u}{\partial x_i}$ denotes the partial derivative of $u$ with respect to $x_i$, and by nondecreasing we mean that $u_{x_i}\geq 0$ for all $i$. We note that when $f=1$ is the uniform density, the solution of \eqref{eq:HJE} is
\begin{equation}\label{eq:exactu}
u(x) = d(x_1\cdots x_d)^\frac{1}{d}.
\end{equation}
 Figure \ref{fig:demo} gives an illustration of this continuum limit when $X_1,\dots,X_n$ are independent and uniformly distributed. The continuum limit in Theorem \ref{thm:main} states that the Pareto fronts converge to the level sets of the viscosity solution $u$ of \eqref{eq:HJE}. While the value of $C_d$ is not needed for sorting, we should mention that it is known only in dimension $d=2$, in which case $C_2=1$~\cite{logan1977,vershik1977}.

Calder et al.~\cite{calder2015PDE} proposed a fast algorithm for approximate nondominated sorting called \emph{PDE-based ranking} that is based on estimating the density function $f$ from a small subset of the data $X_1,\dots,X_n$ and then solving \eqref{eq:HJE} numerically. PDE-based ranking can drastically reduce the computation time of nondominated sorting in low dimensions ($d=2,3$) while maintaining very high sorting accuracy.

Let us say a few words about viscosity solutions. Hamilton-Jacobi equations like \eqref{eq:HJE} generally do not admit classical solutions (i.e., continuously differentiable solutions) due to the possibility of crossing characteristics. There are, however, infinitely many functions $u$ that are differentiable almost everywhere and satisfy \eqref{eq:HJE} at each point of differentiability. The notion of viscosity solution selects from among these infinitely many feasible solutions the one that is `physically correct' for a \emph{very} wide range of problems. The viscosity solution is correct in this context because it captures the continuum limit of the Pareto depth function~\cite{calder2014}. 

The notion of viscosity solution is based on the maximum principle and enjoys very strong stability properties. It is a notion of weak solution that allows merely continuous functions to be solutions of a fully nonlinear PDE. While viscosity solutions may not possess the derivatives appearing in the equation in the classical sense, the reader will not lose much in the way of understanding by assuming that $u$ is continuously differentiable. In the context of viscosity solutions, the maximum principle is used to prove a comparison principle, which says that subsolutions lie below supersolutions provided their boundary conditions do as well. The reason the notion of viscosity solution is correct in this context is that nondominated sorting also obeys a comparison principle; namely, if we introduce new points into our data set (i.e., we increase the point density), then the Pareto depth function increases as well. For more details on viscosity solutions we refer the reader to \cite{bardi1997}. 

\subsection{Anomaly detection}
\label{sec:anomaly}

To illustrate the computational advantages of PDE-based ranking, we consider a concrete application of nondominated sorting to anomaly detection~\cite{hsiao2015}. Anomaly detection refers to the problem of detecting patterns in data that deviate from the expected behavior. It is an important and challenging problem with a wide array of applications, including computer intrusion detection, video surveillance, credit card fraud, and biometrics~\cite{hodge2004survey,chandola2009anomaly}. Many anomaly detection algorithms rely on the availability of a measure of distance (or similarity) between data samples, and look for anomalies by finding samples that are far from their nearest neighbors (see \cite{hsiao2015} and references therein). These algorithms are usually called \emph{similarity-based}, and are widely used due to their simplicity and robustness.

In contrast, feature-based algorithms seek to embed the data into a relatively low dimensional Euclidean space and make use of the ambient Euclidean (or other) distance to detect anomalies. Techniques used for feature-based algorithms include support vector machines (SVM), clustering, neural networks, and statistical approaches based on density estimation~\cite{chandola2009anomaly}. In this paper we consider \emph{similarity-based} approaches.

In many applications, multiple measures of similarity may be required to detect certain types of anomalies. For example, when tracking pedestrians in video surveillance, one criterion may correspond to differences in individual walking speeds, while another might correspond to differences in the shapes of trajectories. Using multiple criteria allows one to detect a wider range of anomalies than could be obtained from a single criterion alone.  

Hsiao et al.~\cite{hsiao2015} proposed an algorithm for multi-criteria anomaly detection that integrates the information from multiple similarity measures via nondominated sorting (or Pareto Depth Analysis (PDA)).  Suppose we have a training set consisting of $N$ objects $Y_1,\dots,Y_N$ and $d$ measures of similarity $c_1,\dots,c_d$ for comparing these objects. Without loss of generality, we assume $0 \leq c_i(\cdot,\cdot) \leq 1$---a lower score indicates the objects are more similar with respect to the $i^{\rm th}$ criteria. The training phase of the algorithm consists of computing the $n:=\binom{N}{2}$ dyads
\begin{equation}\label{eq:dyad}
X_{i,j} = (c_1(Y_i,Y_j),\dots,c_d(Y_i,Y_j)) \in [0,1]^d,
\end{equation}
and constructing the Pareto depth function $U_n$ by applying nondominated sorting to the $n$ points $\{X_{i,j}\}_{i,j=1}^N$. Recall that $U_n(X_{i,j})=k$ if and only if $X_{i,j}$ belongs to the $k^{\rm th}$ Pareto front.

The testing phase of the algorithm receives a new object $Y$ and compares it to all training samples to create $N$ new dyads $Z_1,\dots,Z_N$ given by
\[Z_j = (c_1(Y,Y_j),\dots,c_d(Y,Y_j)).\]
Fix a number $k$, and let $I \subset \{1,\dots,N\}$ denote the indices of training samples that are among the $k$ nearest neighbors of $Y$ with respect to at least one similarity measure $c_i$. The anomaly score for $Y$ is
\begin{equation}\label{eq:anomaly_score}
\nu = \frac{1}{|I|} \sum_{j \in I} U_n(Z_i),
\end{equation}
and $Y$ is declared an anomaly when $\nu$ is larger than a predefined threshold $\rho>0$. We note that it is possible to allow different values of $k$ for each criterion. 
The idea is that nominal samples should be close to many of their nearest neighbors from the training data in one or many similarities, and thus the dyads $Z_1,\dots,Z_N$ will lie on earlier Pareto fronts. An anomalous sample should be far from its nearest neighbors in the training set in many or all of the similarities, and the dyads will consequently fall on deeper fronts.
The PDA anomaly detection algorithm has been validated on real and synthetic data in \cite{hsiao2012,hsiao2015} and has been shown to achieve state of the art results for integrating information from multiple similarities.

\section{New PDE continuum limits}
\label{sec:new}

The Hamilton-Jacobi Equation (HJE) continuum limit \eqref{eq:HJE} gives information about which Pareto front a sample lies on. It is also important in applications to know where a sample lies within its Pareto front, as this gives information about the trade-off between the multiple objectives. We present here some new PDE continuum limits for ordering of points \emph{within} Pareto fronts in dimension two.
\begin{figure}
\centering
\subfigure{\includegraphics[width=0.75\textwidth]{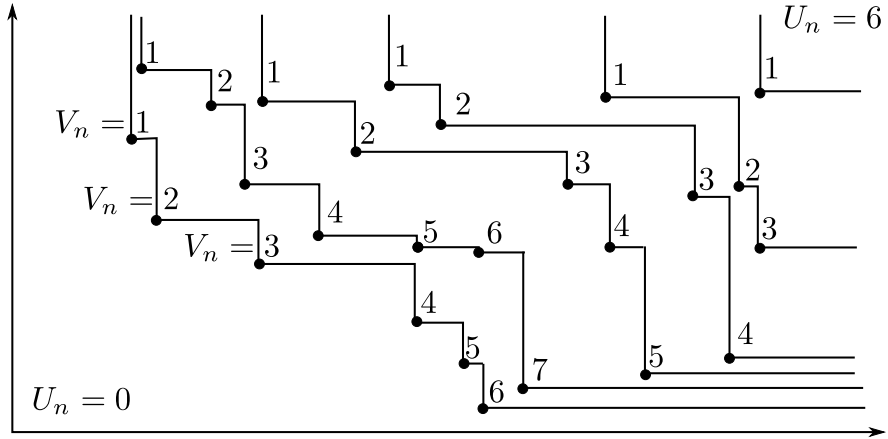}}
\caption{Illustration of the function $V_n$ that orders the points within each nondominated layer.} 
\label{fig:demo_Vn}
\end{figure}

Suppose $d=2$ and let $S_n=\{X_1,\dots,X_n\}$ be a sequence of \emph{i.i.d.}~random variables with continuous density $f$ on $(0,1)^2$. Apply nondominated sorting to $S_n$ and then order the points within each Pareto front by $x_1$-coordinate. This defines a function $V_n:S_n \to \N$ given by
\begin{equation}\label{eq:Vn}
V_n(X_i) := \text{Index of } X_i \text{ within its Pareto front}.
\end{equation}
Figure \ref{fig:demo_Vn} gives an illustration of $V_n$.  

By the continuum limit of nondominated sorting (Theorem 1), there are on the order of $n^\frac{1}{2}$ Pareto fronts. Since there are $n$ points in total, each front should have on the order of $n^\frac{1}{2}$ points. Therefore let us suppose that
\[n^{-\frac{1}{2}}V_n \longrightarrow v \ \ \ \text{ as } \ \  n \to \infty\]
uniformly with probability one, where $v:[0,1]^2 \to \R$ is continuously differentiable.    

Fix a large value of $n$ and consider a point $x\in (0,1)^2$. Fix $\eps>0$ and let
\[y = x + \eps \nabla^\perp u(x),\]
where $\nabla^\perp u = (u_{x_2},-u_{x_1})$. Since $\nabla^\perp u$ is tangent to the level set $\{u=u(x)\}$, we have $u(y)\approx u(x)$, i.e., $x$ and $y$ are roughly on the same Pareto front.  Let $A$ denote the rectangle whose diagonal is the line segment from $x$ to $y$, and let $L_n$ denote the number of points on the Pareto front passing through $x$ and $y$ that fall within $A$. See Figure \ref{fig:alongfronts} for an illustration of the setup. Then we have
\begin{equation}\label{eq:a}
\eps\nabla v(x) \cdot \nabla^\perp u(x) \approx v(y) - v(x)\approx n^{-\frac{1}{2}}(V_n(y) - V_n(x))=n^{-\frac{1}{2}}L_n.
\end{equation}
Here, $\nabla v=(v_{x_1},v_{x_2})$ denotes the gradient of $v$.
When $\eps>0$ is small, the random variables within $A$ are approximately uniformly distributed within $A$. Furthermore, as illustrated in Figure \ref{fig:alongfronts}, we can scale $A$ to the unit box $[0,1]^2$ without changing the partial ordering within $A$. Hence, it is reasonable to conjecture that $L_n \sim c\sqrt{m}$ as $n\to \infty$, where $c>0$ is a universal constant and $m$ is the number of samples falling in $A$. While the value of $c$ is not needed for sorting (since we perform a normalization in \eqref{eq:Wn_app} below), a simple scaling argument suggests that $c=1$, so we will take $L_n \sim \sqrt{m}$. By the law of large numbers
\[m \sim n\int_A f\, dx \approx n|A|f(x)= n\eps^2u_{x_1}u_{x_2}f(x),\]
since the side lengths of $A$ are $|x_1-y_1|=\eps u_{x_2}$ and $|x_2-y_2|=u_{x_1}$.
Combining this with $u_{x_1}u_{x_2} = f$, \eqref{eq:a} and $L_n \sim \sqrt{m}$ yields
\[\eps\nabla v(x) \cdot \nabla^\perp u(x) \approx f(x)\eps.\]
Hence this simple heuristic argument suggests that $v$ satisfies the linear transport equation
\begin{equation}\label{eq:T}
\boxed{\left\{\begin{aligned}
\nabla v \cdot \nabla^\perp u &= f&&\text{in } (0,1)^2\\
v&=0&&\text{on } \{x_2 = 1\}.
\end{aligned}\right.}
\end{equation}
Recall $u$ is the viscosity solution of \eqref{eq:HJE}. When $f=1$ we have $u(x)=2\sqrt{x_1x_2}$. If we plug this into \eqref{eq:T} and look for a separable solution of the form $v(x)=f_1(x_1)f_2(x_2)$ we find that when $f=1$
\begin{equation}\label{eq:exactv}
v(x)  = -\log(x_2)\sqrt{x_1x_2}.
\end{equation}
\begin{figure}
\centering
\boxed{\includegraphics[width=0.8\textwidth]{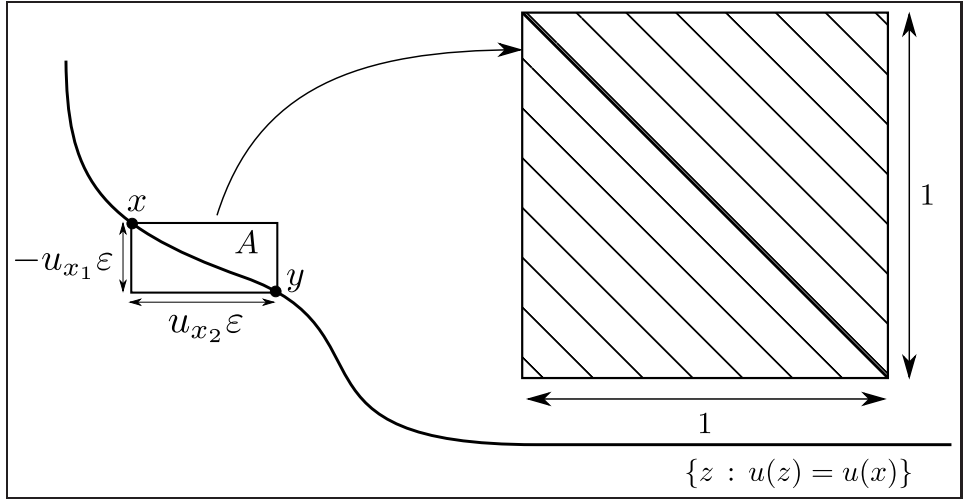}}
\caption{A depiction of some quantities from the derivation of the new continuum limit PDE.}
\label{fig:alongfronts}
\end{figure}

Since each Pareto front has in general a different number of points, the values of $V_n$ within different fronts are difficult to compare. Therefore it is natural to consider the following normalization:
\begin{equation}\label{eq:Wn_app}
W_n(X_i) := \frac{V_n(X_i)}{\# \F(X_i)},
\end{equation}
where $\F(X_i)$ denotes the Pareto front that $X_i$ belongs to. The quantity $W_n(X_i)$ is an index between 0 and 1 that gives information about where the point $X_i$ falls within its Pareto front.  The arguments above suggest that 
\[W_n \longrightarrow w \ \ \text{ uniformly with probability one, where}\]
\begin{equation}\label{eq:w}
w(x) = \frac{v(x)}{v(1,\psi(u(x))},
\end{equation}
and $\psi$ is the inverse of $x_2 \mapsto u(1,x_2)$. In other words, we are normalizing $v(x)$ by the asymptotic number of points on the front to which $x$ belongs. 

The expression in \eqref{eq:w} is difficult to work with numerically. We will instead derive a PDE for $w$. Differentiating \eqref{eq:w} yields
\[v\nabla w  = w\nabla v - w^2v_{x_2}\psi'(u)\nabla u.\]
Take the dot product of both sides with $\nabla^\perp u$ and recall \eqref{eq:T} to find that
\[v\nabla w \cdot \nabla^\perp u = w\nabla v \cdot \nabla^\perp u = wf.\]
Since $w=1$ on $\{x_1=1\}$, $w$ can be characterized as the solution of the following transport equation
\begin{equation}\label{eq:TW}
\boxed{\left\{\begin{aligned}
v\nabla w \cdot \nabla^\perp u &= wf&&\text{in } (0,1)^2\\
w&=1&&\text{on } \{x_1 = 1\}.
\end{aligned}\right.}
\end{equation}
We note that it would seem equally reasonable to have chosen the boundary condition $w=0$ on $\{x_2=1\}$ instead. However, in this case it is easy to verify that $w=v$ would solve \eqref{eq:TW}, so the solution is not uniquely determined by the boundary condition $w=0$ on $\{x_2=1\}$. This issue arises numerically as well. Indeed, we have found experimentally that if we solve \eqref{eq:TW} numerically with an upwind scheme and the boundary condition $w=0$ on $\{x_2=1\}$ we find the solver automatically computes $v$ instead of $w$, and it is unclear how to select the correct solution without changing the boundary condition to $w=1$ on $\{x_1=1\}$. It is impossible to specify both boundary conditions simultaneously since the characteristic curves, which are the level curves of $u$, flow through both boundaries.

Note that when $f=1$ we have $u(x)=2\sqrt{x_1x_2}$ and $v(x)=-\log(x_2)\sqrt{x_1x_2}$. Plugging these into \eqref{eq:TW} and using the method characteristics we find that for $f=1$
\begin{equation}\label{eq:wexact}
w(x) = \frac{\log(x_2)}{\log(x_1) + \log(x_2)}.
\end{equation}

See Figure \ref{fig:demo2} for a comparison of $V_n$ and $W_n$ to their continuum limits \eqref{eq:T} and \eqref{eq:TW}, respectively. While the arguments in this section are not rigorous, we present a convergence analysis in Section \ref{sec:conv_anal} that gives very strong numerical evidence for their validity. A rigorous proof is the subject of current investigation and is far outside the scope of this paper.
\begin{figure}
\centering
\subfigure[$V_n$ vs.~$v$]{\includegraphics[clip=true,trim=20 20 20 20,width=0.45\textwidth]{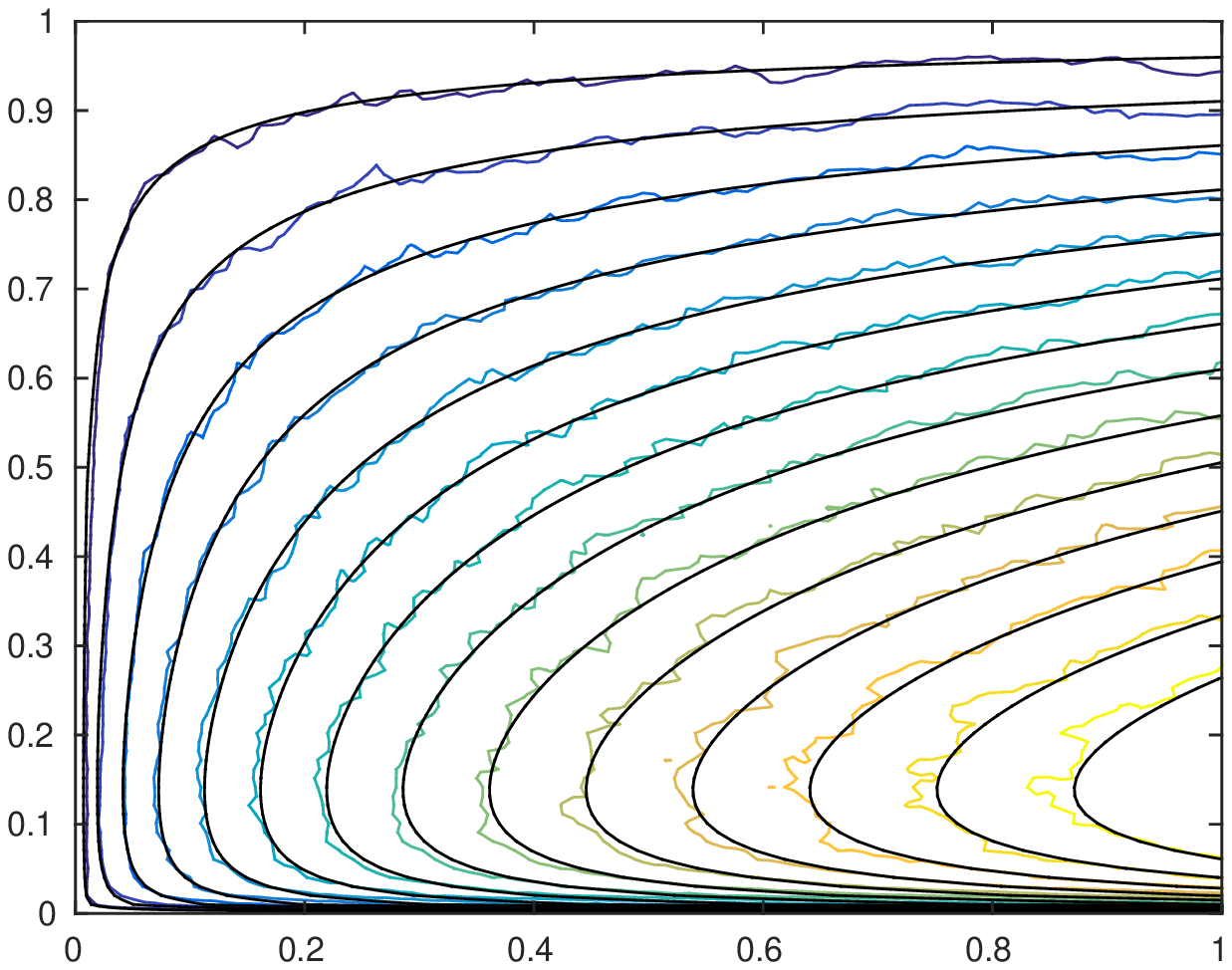}}
\subfigure[$W_n$ vs.~$w$]{\includegraphics[clip=true,trim=20 20 20 20,width=0.45\textwidth]{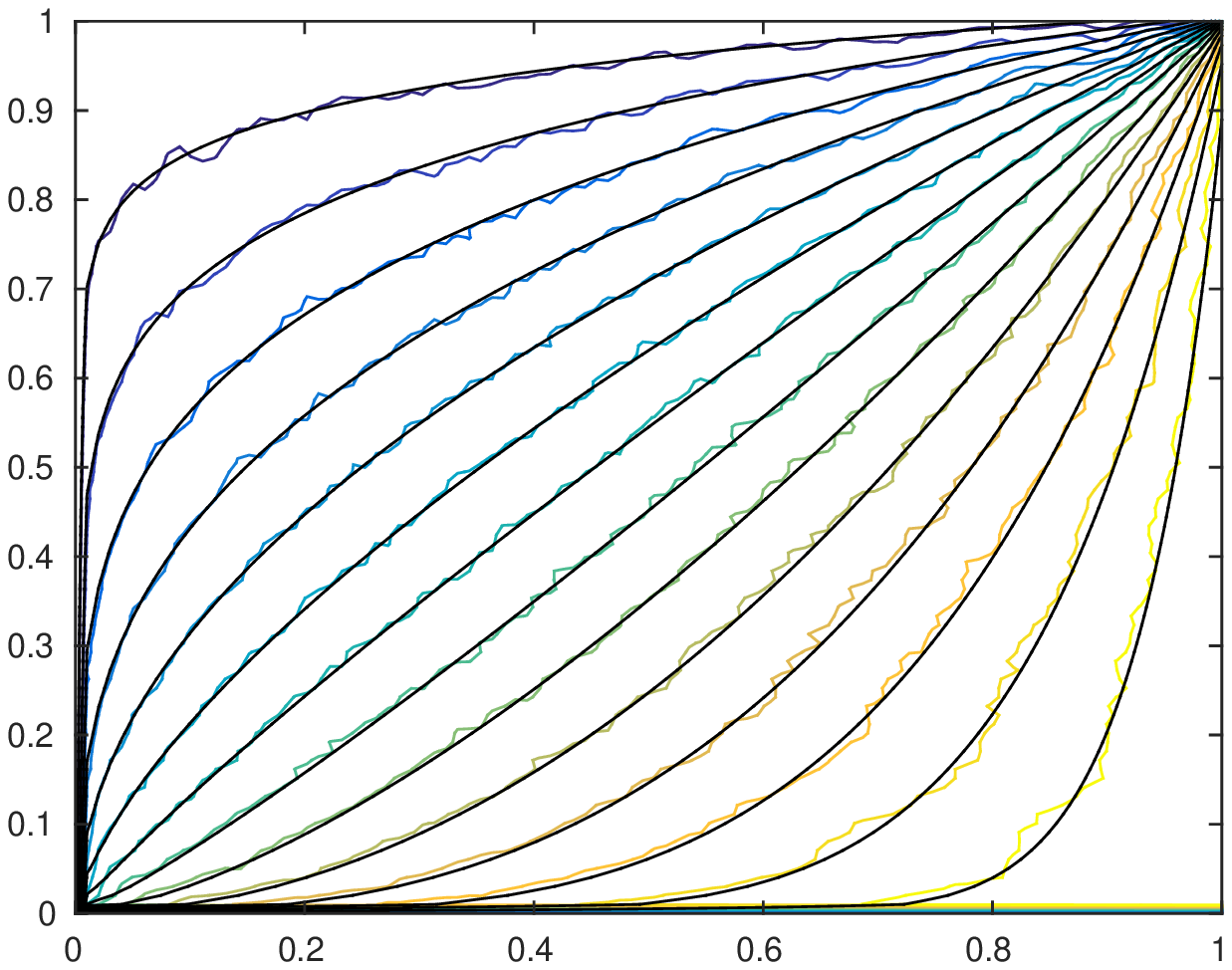}}
\caption{Results of a simulation comparing the level sets of $V_n$ and $W_n$ to the solutions of their continuum limits PDEs \eqref{eq:T} and \eqref{eq:TW}.} 
\label{fig:demo2}
\end{figure}

On a more technical note, since $u$ is not necessarily differentiable the velocity $\nabla^\perp u$ in the transport equations \eqref{eq:T} and \eqref{eq:TW} is not well-defined. To make sense of these PDE rigorously, we can instead write them in divergence form, since
\[-\text{div}(u\nabla^\perp v) = \nabla v \cdot \nabla^\perp u.\]
This suggests that it is possible to prove existence and uniqueness of weak solutions (defined via integration by parts) of \eqref{eq:T} in the Sobolev space $H^1$, under the requirement that $u$ is merely continuous. Such results are outside the scope of this paper, and we intend to pursue them in a future work.

\subsection{Convergence analysis}
\label{sec:conv_anal}

We present here a convergence analysis for the continuum limits \eqref{eq:T} and \eqref{eq:TW} in the case that $f\equiv 1$, i.e., the samples are independent and uniformly distributed on the unit box $[0,1]^2$. In this case we can solve all three PDEs \eqref{eq:HJE}, \eqref{eq:T}, and \eqref{eq:TW} in closed form using the formulas \eqref{eq:exactu}, \eqref{eq:exactv}, and \eqref{eq:wexact}, respectively.

We performed a convergence analysis by drawing $X_1,\dots,X_n$ independent and uniformly distributed on $[0,1]^2$ and computing $V_n$ and $W_n$ according to their definitions \eqref{eq:Vn} and \eqref{eq:Wn_app}, respectively. We measured the discrepancy with the continuum limits in the $\ell_1$ and $\ell_\infty$ norms, computed by
\[\|v - n^{-\frac{1}{2}}V_n\|_{\ell_1} := \frac{1}{n}\sum_{i=1}^n |v(X_i) - n^{-\frac{1}{2}}V_n(X_i)|.\]
and
\[\|v - n^{-\frac{1}{2}}V_n\|_{\ell_\infty} := \max_{1 \leq i \leq n} |v(X_i) - n^{-\frac{1}{2}}V_n(X_i)|,\]
respectively.
The definitions of $\|w - W_n\|_{\ell_\infty}$ and $\|w - W_n\|_{\ell_1}$ are similar. Figure \ref{fig:conv} shows the errors for a single realization and various values of $n$ ranging from  $n=10^2$ to $n=10^8$.  Each of the errors is observed to be converging to zero at a rate of $O(n^{-\alpha})$ where $\alpha \approx 0.25$, except for $\|w - W_n\|_{\ell_\infty}$. Upon closer inspection, the function $w(x) = \log(x_2)/(\log(x_1) + \log(x_2))$ is discontinuous at $x=(1,1)$, hence uniform convergence is impossible. This discontinuity reflects the fact that near $x=(1,1)$ the Pareto fronts cut off an infinitesimal portion of the top corner of the box, and hence $W_n$ transitions from $0$ to $1$ over an infinitesimally short distance.

\begin{figure}
\centering
\subfigure[$n^{-\frac{1}{2}}V_n$ vs $v$]{\includegraphics[clip=true, trim = 40 0 40 20, width=0.45\textwidth]{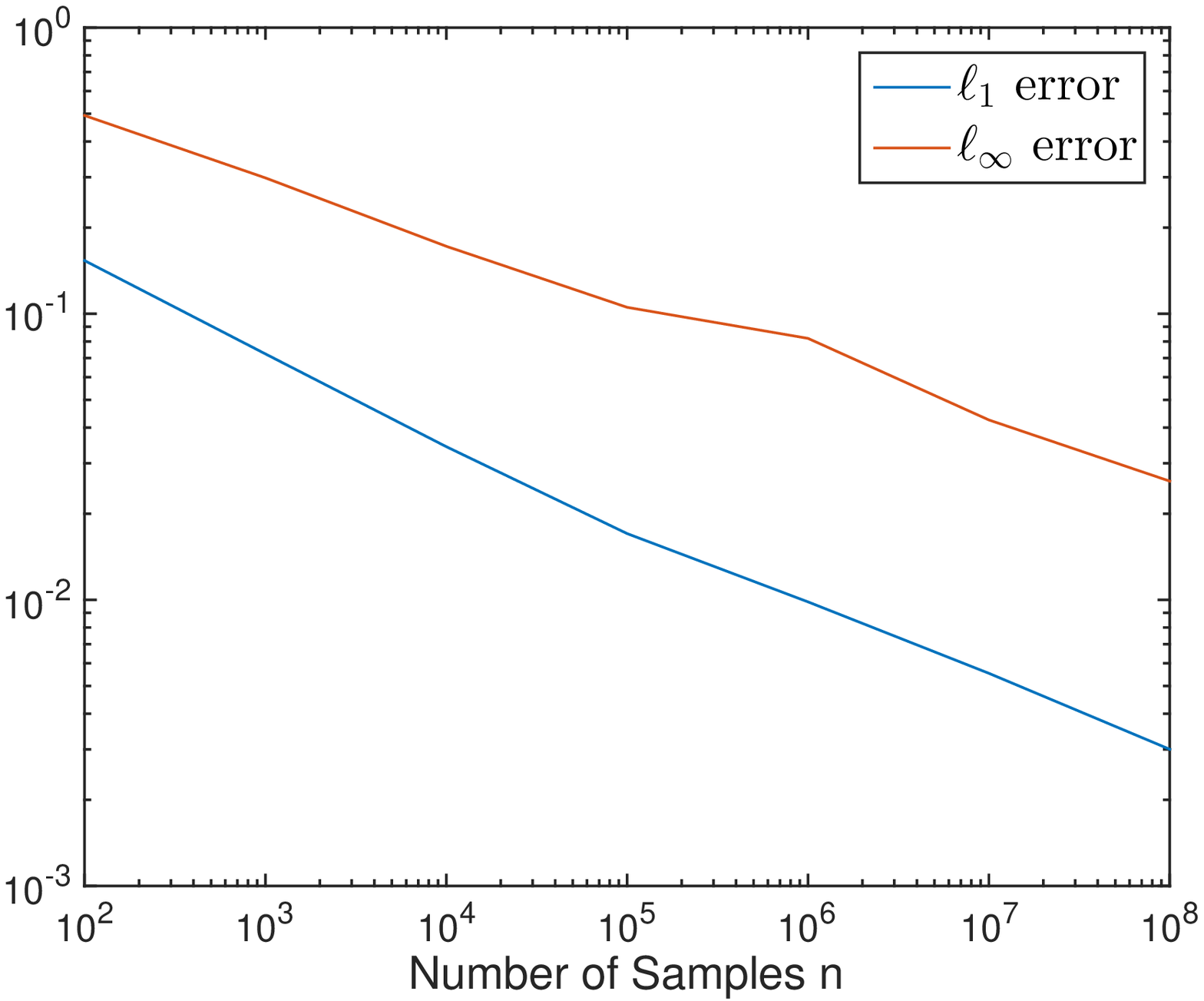}}
\subfigure[$W_n$ vs $w$]{\includegraphics[clip=true, trim = 40 0 40 20, width=0.45\textwidth]{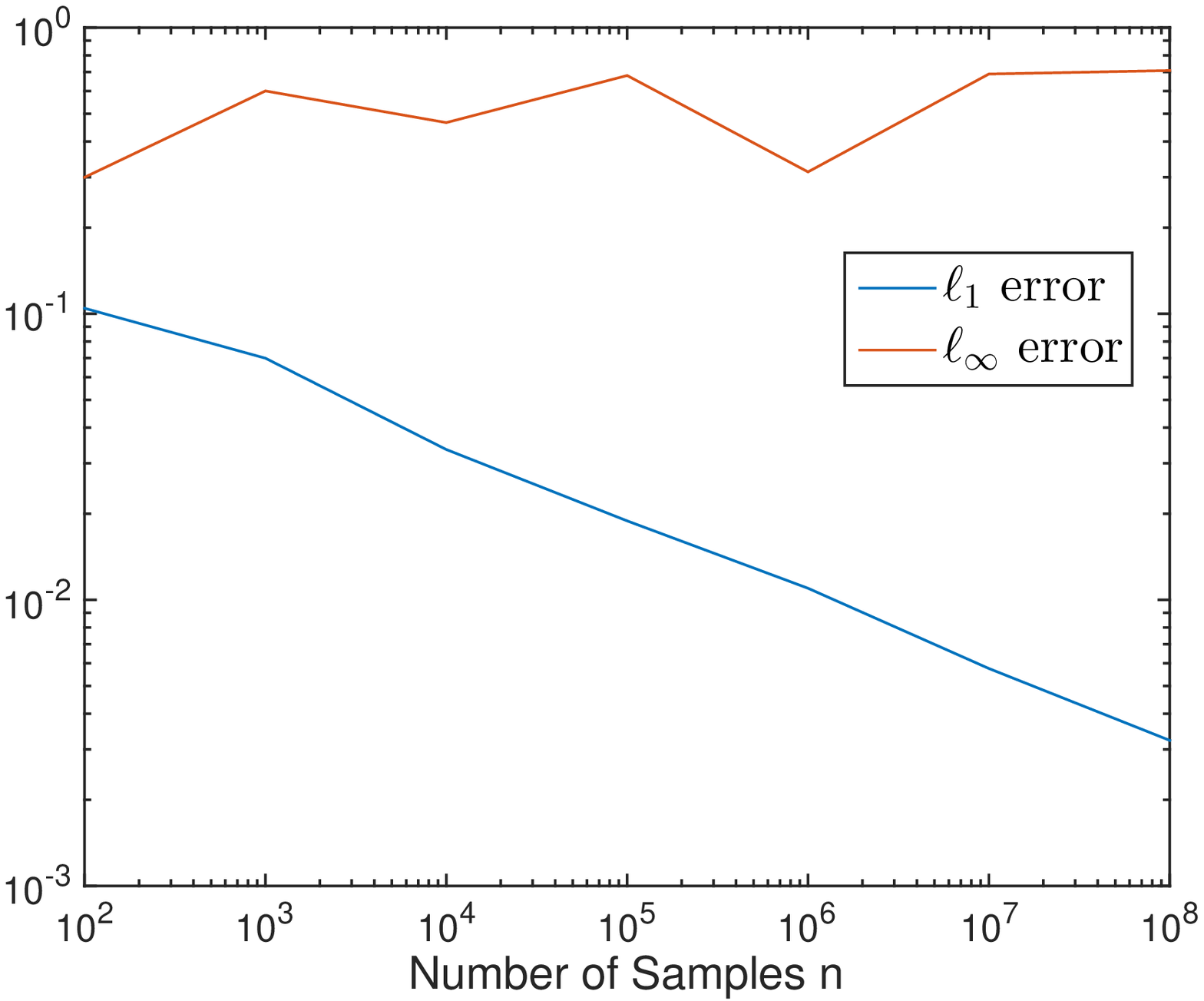}}
\caption{Convergence analysis of conjectured continuum limits. In (a), the error is consistent with $O(n^{-\alpha})$ where $\alpha\approx 0.3$ for the $\ell_1$ norm, and $\alpha\approx 0.2$ for the $\ell_\infty$ norm. In (b), the error for the $\ell_1$ norm is consistent with $O(n^{-\alpha})$ where $\alpha\approx 0.25$. Since $w$ is discontinuous at $x=(1,1)$, the convergence $W_n \to w$ \emph{cannot} be uniform (i.e., in the $\ell_\infty$ norm).}
\label{fig:conv}
\end{figure}

\section{Numerical schemes}
\label{sec:schemes}

We address in this section the problems of solving the PDEs \eqref{eq:HJE}, \eqref{eq:T}, \eqref{eq:TW} numerically. Since each PDE involves the solution of the previous PDEs, they must be solved in the order \eqref{eq:HJE}-\eqref{eq:T}-\eqref{eq:TW}. We solve \eqref{eq:HJE} numerically using the (formally) first order accurate upwind scheme presented in \cite{calder2015numerical}. The scheme is similar to fast sweeping, but only requires sweeping the grid exactly once, and thus has linear complexity in the number of grid points. 

Here, we show how to construct upwind finite difference schemes for the transport PDEs \eqref{eq:T} and \eqref{eq:TW}. Fix a grid resolution $h>0$ and for $U \subset \R^2$ define $U_h:=U\cap (h\Z^2)$.  We will solve the transport equations on the grid $[0,1]^2_h$. Both PDEs are degenerate when $\nabla u = 0$, which can only happen when $f$ vanishes (since $u_{x_1}u_{x_2}=f$). To avoid this degeneracy, we numerically precondition the density by replacing $f$ with $f + h^2$ before solving \eqref{eq:HJE} numerically.

The transport equation \eqref{eq:T} can be written out as
\[u_{x_2}v_{x_1} - u_{x_1}v_{x_2} = f \ \ \text{on } (0,1)^2\]
with $v(x_1,1)=0$. The unknown function is $v$; $u$ is obtained by solving \eqref{eq:HJE}. Since $u_{x_1}\geq 0$ and $u_{x_2}\geq 0$ the coefficients of $v_{x_1}$ and $v_{x_2}$ have opposite signs. Thus an upwind scheme will use either (A) backward differences for $v_{x_1}$ and forward differences for $v_{x_2}$, or (B) vice-versa. The choice depends on the direction we want information to propagate. Since our boundary condition is $v=0$ on $\{x_2=1\}$ and information flows along Pareto fronts in the positive $x_1$ direction and negative $x_2$ direction, the correct choice for an upwind scheme is (A) backward differences in $v_{x_1}$ and forward differences in $x_2$, which effectively forces the scheme to look backwards along the current Pareto front (or level set of $u$).

Thus, our upwind scheme for \eqref{eq:T} is
\begin{equation}\label{eq:T_scheme}
\boxed{\left\{\begin{aligned}
u_{x_2}D^-_1v_h - u_{x_1}D^+_2 v_h&=f&&\text{in } (0,1)^2_h,\\
v_h(0,x_2)=v_h(x_1,1)&=0.&&
\end{aligned}\right.}
 \end{equation}
where $v_h:[0,1]^2_h\to \R$ is the numerical solution, and $D^{\pm}_i$ are the finite differences defined by
\[D^{\pm}_iv(x) := \pm \frac{v(x\pm he_i) - v(x)}{h},\]
where $e_1=(1,0)$ and $e_2=(0,1)$.  At each grid point, \eqref{eq:T_scheme} is linear equation that is readily solved for $v_h(x)$ in terms of $v_h(x-he_1)$ and $v_h(x+he_2)$ to obtain
\begin{equation}\label{eq:vscheme}
v_h(x) = \frac{u_{x_2}(x)v_h(x-he_1) + u_{x_1}(x)v_h(x+he_2) + hf(x)}{u_{x_1}+u_{x_2}}.
\end{equation}
The numerical solution $v_h$ is computed by sweeping the grid $(0,1)^2_h$ exactly once in the upwind direction $(1,-1)$ starting with the boundary condition $v_h(0,x_2)=v_h(x_1,1)=0$. We note that the boundary condition $v_h(0,x_2)=0$ is just for numerical convenience, and is not used directly by the scheme, since $u_{x_2}(0,x_2)=0$ so on the line $x_1=0$ so the solution depends only on $v_h(x+he_2)$ when $x_1=0$.  When computing $v_h(x)$, we replace $u_{x_1}$ and $u_{x_2}$ in \eqref{eq:vscheme} by first order finite differences of the solution $u_h$ of \eqref{eq:HJE} on the same grid. Our numerical experiments suggest that the scheme is not sensitive to the choice of discretization of $u_{x_1}$ and $u_{x_2}$. We note that convergence of the scheme \eqref{eq:T_scheme} is a classical result when $u \in C^1$, however, $u$ is in general only H\"older continuous. We leave the analysis of the scheme for $u\not\in C^1$ to future work.

We now consider the second transport equation \eqref{eq:TW}. If we again write the PDE out we have
\[vu_{x_2}w_{x_1} - vu_{x_1}w_{x_2} = wf \ \ \text{in } (0,1)^2\]
with boundary condition $w(1,x_2)=w(x_1,0)=1$. Since $vu_{x_2}\geq 0$ and $vu_{x_1}\geq 0$ we have the same upwind choices as before. However, here we want information to propagate from the boundary where $x_1=1$ or $x_2=0$ backwards along Pareto fronts (level sets of $u$) in the direction $(-1,1)$. Thus we use forward differences for $w_{x_1}$ and backward differences for $w_{x_2}$, and our upwind scheme for \eqref{eq:TW_scheme} has the form
\begin{equation}\label{eq:TW_scheme}
\boxed{\left\{\begin{aligned}
v_hu_{x_2}D^+_1w_h - v_hu_{x_1}D^-_2 w_h&=w_h f&& \text{in } (0,1)^2_h,\\
w_h(1,x_2) = w_h(x_1,0) &=1.&&
\end{aligned}\right.}
\end{equation}
This is a linear equation that can be solved  for $w_h(x)$ in terms of $w_h(x+he_1)$ and $w_h(x-he_2)$ to obtain
\begin{equation}\label{eq:wscheme}
w_h(x) = \frac{v_h(x)u_{x_2}(x)w_h(x+he_1) + v_h(x)u_{x_1}(x)w_h(x-he_2)}{hf(x) + v_h(x)u_{x_2}(x) + v_h(x)u_{x_1}(x)}.
\end{equation}
This scheme can also be solved in a fast sweeping pattern visiting each grid point exactly once, and thus has linear complexity. Note that the boundary condition $w_h(x_1,0)=1$ is not directly used by the scheme, since $u_{x_1}(x_1,0)=0$. We impose the condition simply for convenience. We note here as well that convergence of the scheme is classical when $u \in C^1$, and we leave the case of $u \not\in C^1$ to future work.

\section{Real-time anomaly detection}
\label{sec:algorithm}

We propose here a modification of the PDA multicriteria anomaly detection algorithm~\cite{hsiao2015} to the setting of online streaming data. Suppose we have a stream of possibly nonstationary data $\{Y_t\}_{t \in \N}$, and $d$ measures of similarity $c_1,\dots,c_d$ for comparing data samples. As before we suppose that $0 \leq c_i(\cdot,\cdot) \leq 1$. In the streaming setting, we observe the data $Y_t$ sequentially and must determine whether $Y_t$ is an anomaly based only on the previous history $\{Y_s \, : \, s< t\}$. Due to memory and computational constraints, it may not be feasible to use this entire history, especially when $t$ is large. Therefore, we fix $T\geq 1$ and consider the windowed history
\begin{equation}\label{eq:Ht}
H_t = \{Y_s \, : \, t-T \leq s \leq t-1\}.
\end{equation}
We use the history $H_t$ as training data in order to determine whether $Y_t$ is an anomaly. Even without memory constraints, only the recent history $H_t$ can be considered reliable when the data is nonstationary.   As before, we define dyads
\[X_{r,s} = (c_1(Y_r,Y_s),\dots,c_d(Y_r,Y_s)) \in [0,1]^d\]
corresponding to every pair $(Y_r,Y_s)$ of the data stream. If we use the PDA anomaly detection algorithm with exact nondominated sorting, then we would need to store in memory all of the $n:=\binom{T}{2}=O(T^2)$ dyads corresponding to pairs from the history $H_t$. Since the addition of a single new sample can potentially affect the arrangement of \emph{all} the Pareto fronts, re-training the model when new samples are acquired requires applying nondominated sorting to \emph{all} $O(T^2)$ dyads, which has complexity slightly worse than $O(T^2)$ for memory and operations. This makes it impossible to update the model frequently in the streaming setting without considering some type of approximation to the sorting.

Using PDE-based ranking we can reduce this complexity to $O(T)$. We keep a running estimate of the marginal distribution of the dyads using the following kernel density estimator
\begin{equation}\label{eq:ker}
f_t(x) = \frac{1}{nh^d} \sum_{t-T \leq r < s \leq t-1} K\left(\frac{x-X_{r,s}}{h}\right).
\end{equation}
Although there are $O(T^2)$ terms in the sum above, the density estimation $f_t(x)$ can be updated recursively in $O(T)$ time by writing $f_t(x) = f_{t-1}(x) + g_t(x)$ where
\begin{equation*}
g_t(x) = \frac{1}{nh^d} \sum_{s=t-T}^{t-1}K\left(\tfrac{x-X_{s,t}}{h}\right) - K\left(\tfrac{x-X_{(t-T-1),s}}{h}\right).
\end{equation*}
In our experiments, we use a simple histogram estimator, which is a special case of \eqref{eq:ker}. We then compute an approximation $U_t$ of the Pareto depth function by solving the HJE \eqref{eq:HJE} numerically using the estimated density $f_t$. By Theorem \ref{thm:main} the continuum approximation of the anomaly score is
\begin{equation}\label{eq:approxscore}
\nu_t = \frac{1}{|I_t|}\sum_{s\in I_t} U_t(X_{s,t}),
\end{equation}
where $I_t\subset \{t-T,\dots,t-1\}$ denotes the indices of samples from the history $H_t$ that are among the $k$ nearest neighbors of $Y_t$ with respect to $c_i$ for at least one $i \in \{1,\dots,d\}$. We declare $Y_t$ anomalous if $\nu_t$ is greater than a predefined threshold $\rho$. The steps above work in arbitrary dimension $d\geq 2$ as well.

For the anomaly classification, we specialize to the case of $d=2$. If $Y_t$ is declared an anomaly, we then solve the transport equations \eqref{eq:T} and \eqref{eq:TW} using the schemes \eqref{eq:T_scheme} and \eqref{eq:TW_scheme}, respectively, to obtain $W_t:=w_h$. We define the \emph{anomaly classification score}
\begin{equation}\label{eq:class_score}
\mu_t = \frac{1}{|I_t|}\sum_{s\in I_t} W_t(X_{s,t}),
\end{equation}
and we declare $Y_t$ a $c_1$-anomaly if $\mu_t > 0.5$, and a $c_2$-anomaly if $\mu_t<0.5$. The idea is that if a sample is a $c_1$-anomaly, then the first coordinate of the dyads $c_1(Y_s,Y_t)$ should be larger on average than in the training set, which is our windowed history. Therefore, the dyads corresponding to $Y_t$ will be on average further to the right along each Pareto front. This corresponds to a front index larger than $0.5$ on average. The situation is similar for $c_2$ anomalies, except that the dyads will be biased towards the left side of the fronts. See Algorithm 1 for a summary of our algorithm in pseudocode.

\begin{algorithm}[h!]
\caption{PDE-based online anomaly detection}
\label{alg}
\begin{algorithmic}[1]
\STATE \textbf{Given:} $\rho>0$ and $T\in \N$
\STATE $f_T \gets$ \eqref{eq:ker} \COMMENT{Initialize density}
\FOR{$t = T+1 \to \infty$}
   \STATE $f_t \gets f_{t-1} + g_t$ \COMMENT{Update density estimation}
   \STATE $U_t \gets$\eqref{eq:HJE}  \COMMENT{Solve HJE continuum limit}
   \STATE $\nu_t \gets$  \eqref{eq:approxscore} \COMMENT{Compute anomaly score}
   \IF{$\nu_t > \rho$} 
      \STATE Declare $Y_t$ to be anomalous
      \STATE $W_t \gets$ \eqref{eq:TW_scheme} \COMMENT{Solve transport equations}
      \STATE $\mu_t \gets$  \eqref{eq:class_score} \COMMENT{Compute anomaly classification score}
      \STATE \textbf{if} $\mu_t > 0.5$ \textbf{ then } $Y_t$ is a $c_1$-anomaly
      \STATE \textbf{if} $\mu_t \leq 0.5$ \textbf{ then } $Y_t$ is a $c_2$-anomaly
   \ENDIF
\ENDFOR
\end{algorithmic}
\end{algorithm}

There are some obvious modifications we could make to improve the performance of the algorithm. First, the continuum limit PDE need not be solved at every iteration, and could instead be solved only periodically, or whenever the density estimation $f_t$ has substantially changed. Second, to keep track of a larger history without incurring additional costs, the history $H_t$ could contain $T$ elements equally (or randomly) spaced among the previous $\alpha T$ samples, where $\alpha \gg 1$. The algorithm would remain otherwise unchanged and the complexity of each iteration remains $O(T)$. 

Since the PDE-based anomaly detection algorithm is based on continuum approximations, it is natural to seek a quantification of the approximation error. If we assume the dyads are \emph{i.i.d.}~we can prove the following convergence rate.
\begin{theorem}[Convergence Rate]\label{thm:rate}
Let $X_1,\dots,X_n$ be \emph{i.i.d.}~with a Lipschitz continuous probability density $f:[0,1]^2 \to [m,\infty)$, where $m>0$. For $h>0$ let $\hat{u}_h$ denote the numerical solution of \eqref{eq:HJE} obtained via estimating $f$ from $X_1,\dots,X_n$ via a histogram aligned to the grid of spacing $h$ on $[0,1]^2$. Then there exist constants $C_1,C_2>0$ such that 
\begin{equation}\label{eq:bound}
\max_{[0,1]^2_h}|\hat{u}_h - u| \leq C_1\sqrt{h}
\end{equation}
holds with probability at least $1-\exp(-C_2nh^5-2\log(h))$, where $u$ is the viscosity solution of \eqref{eq:HJE}.
\end{theorem}
Theorem \ref{thm:rate} suggests we should choose $h$ as a function of $n$ so that $nh^5 \gg \log(h^{-1})$. In particular, if we choose $h=h(n)\to 0^+$ so that
\[\lim_{n\to \infty} \frac{nh^5}{\log(n)} = \infty,\]
then by the Borel-Cantelli Lemma  $\hat{u}_h \to u$ almost surely and uniformly on $[0,1]^2$ as $n\to \infty$. We also note that Theorem \ref{thm:rate} extends easily to higher dimensions $d\geq 3$. In this case the same convergence rate \eqref{eq:bound} holds with probability at least
\[1-\exp(-C_2nh^{2d+1}-d\log(h)).\]
\vspace{-2mm}
\begin{proof}[Proof of Theorem \ref{thm:rate}]
Let $X_1,\dots,X_n$ be independent and identically distributed random variables on $[0,1]^2$ with Lipschitz continuous density $f:[0,1]^2 \to [0,\infty)$. Recall that $f$ is assumed to be positive on $[0,1]^2$, i.e., there exists $m>0$ such that $f \geq m$. Let $h:=1/K>0$ be the grid resolution for solving \eqref{eq:HJE} numerically and for estimating the density $f$ with a histogram estimator, where $K\in \N$. For $1 \leq i \leq k$ and $1 \leq j \leq K$ let
\[B_{ij} = [(i-1)h,ih)\times[(j-1)h,jh)\]
denote the grid cell corresponding to $(i,j)$, and let $N_{ij}$ denote the number of samples from $X_1,\dots,X_n$ falling in $B_{ij}$. Then $N_{ij}$ is a Binomial random variable with parameters $n$ and
\[p_{ij} = \int_{B_{ij}} f(x) \, dx.\]
By the Chernoff-Hoeffding bound (see, e.g., \cite{dubhashi2009concentration}) we have
\begin{equation}\label{eq:chernoff}
\P\left(|N_{ij} - \E N_{ij}| \geq t\right) \leq \exp\left( \frac{-2t^2}{n}\right)
\end{equation}
for all $t\geq 0$. Let $x_{ij}=(ih,jh)$ denote the grid points. The histogram estimation of $f$ at grid point $x_{ij}$ is given by
\[\hat{f}_h(x_{ij}) := \frac{N_{ij}}{n|B_{ij}|} = \frac{N_{ij}}{nh^2}.\]
Combining this with \eqref{eq:chernoff} we have
\begin{equation}\label{eq:chernoff2}
\P\left(\left|\hat{f}_h(x_{ij}) - \E \hat{f}_h(x_{ij})\right| \geq t\right) \leq \exp\left( -2nh^4t^2\right)
\end{equation}
Since $\E \hat{f}_h(x_{ij}) = p_{ij}/h^2$ we have
\begin{equation}\label{eq:calc1}
\left|f(x_{ij}) - \E \hat{f}_h(x_{ij})\right| = \frac{1}{h^2}\left|\int_{B_{ij}} f(x_{ij}) - f(x) \, dx\right| \leq \frac{1}{h^2} \int_{B_{ij}} |f(x_{ij}) - f(x)| \, dx \leq Ch,
\end{equation}
due to the fact that $f$ is Lipschitz. Here, $C$ depends on the Lipschitz constant of $f$, which is defined by
\[\lip(f) = \sup_{x\neq y} \frac{|f(x) - f(y)|}{|x-y|}.\]
It follows from \eqref{eq:calc1} that
\begin{align*}
|f(x_{ij}) - \hat{f}_h(x_{ij})| &\leq  \left|\hat{f}_h(x_{ij}) - \E\hat{f}_h(x_{ij})\right| + \left|f(x_{ij}) - \E\hat{f}_h(x_{ij})\right|  \\
&\leq \left|\hat{f}_h(x_{ij}) - \E\hat{f}_h(x_{ij})\right|  + Ch.
\end{align*}
Combining this with \eqref{eq:chernoff2} and the union bound, there exists $C_1>0$ such that
\begin{equation}\label{eq:chernoff3}
\P\left(\|\hat{f}_h - f\|_\infty \geq \lambda\right) \leq \exp\left(-2nh^4(\lambda - C_1h)^2 - 2\log(h)\right),
\end{equation}
for all $\lambda > C_1h$, where
\[\|u - v\|_\infty := \max_{x_{ij} \in [0,1]^2_h} |u(x_{ij}) - v(x_{ij})|.\]

Let $u_h$ and $\hat{u}_h$ denote the numerical solutions of \eqref{eq:HJE} on the grid of spacing $h>0$ computed with $f$ and $\hat{f}_h$ on the right hand side, respectively. Standard maximum principle arguments (see \cite{calder2015numerical}) yield
\begin{equation}\label{eq:maxp}
\|\hat{u}_h - u_h\|_\infty \leq C\|\hat{f}_h - f\|_\infty,
\end{equation}
where the constant $C$ depends on the lower bound $m>0$ on $f$. By \cite[Theorem 1,2]{calder2015numerical}, there exists a constant $C>0$ such that
\[\|u_h - u\|_\infty \leq C\sqrt{h}.\]
Combining this with \eqref{eq:maxp} we have
\[\|\hat{u}_h - u\|_\infty \leq C_2(\|\hat{f}_h - f\|_\infty +  \sqrt{h}),\]
for some $C_2>0$.
By \eqref{eq:chernoff3}    
\[\P\left(\|\hat{u}_h - u_h\|_\infty \geq C_2(\lambda + \sqrt{h}) \right) \leq \exp\left(-2nh^4(\lambda - C_1h)^2 - 2\log(h)\right).\]
Setting $\lambda = C\sqrt{h}$ for large enough $C$ we have
\[\P\left(\|\hat{u}_h - u\|_\infty \geq C_3\sqrt{h} \right) \leq \exp\left(-C_4nh^5 - 2\log(h)\right),\]
for all $0 < h \leq 1$.
\end{proof}

\section{Numerical results}
\label{sec:numerics}

We present several experiments that provide numerical evidence supporting the above arguments and outlining the effectiveness of our algorithm. The first two experiments were performed using synthetic data from \cite{hsiao2012,hsiao2015}. The streaming experiments consist of $1500$ total samples with a window history of $T=500$. To underscore the adaptive nature of our algorithm, each of these experiments incurs a significant trend change in the middle of the stream. The third and final experiment was performed with a real pedestrian trajectory data set from a video surveillance problem.

To evaluate the performance of the streaming algorithm we use a Receiver Operating Characteristic (ROC) curve and its resulting Area Under the Curve (AUC). We consider how the AUC varies with time as the algorithm takes in points from a stream. When changing the trend in the simulated streams, we also accordingly change the data used to generate the ROC curves, thereby giving us an appropriate method to visualize the learning aspect of the algorithm. In each simulated data stream we evaluate both the anomaly detection and anomaly classification. The results presented below represent the average of 20 trials. All PDEs were solved on a $100\times100$ grid.

In each experiment, we compare our continuum limits against the exact sorting PDA algorithm from \cite{hsiao2012,hsiao2015} and we see little to no difference in anomaly detection performance. The PDE-approximations reduce the complexity by an order of magnitude---from $O(T^2)$ to $O(T)$. To give an idea of the difference in CPU time, each trial in the experiments below takes 27 seconds to process with the PDE-approximations, compared to 413 seconds with exact sorting. If we increase the stream length and data history $T$ by a factor of 3, the PDE-approximations take 160 seconds, while the exact sorting PDA algorithm takes over 9.3 hours.

\begin{figure}[h]
\centering
\subfigure[]{\includegraphics[clip=true,trim=15 0 20 20,width=0.5\textwidth]{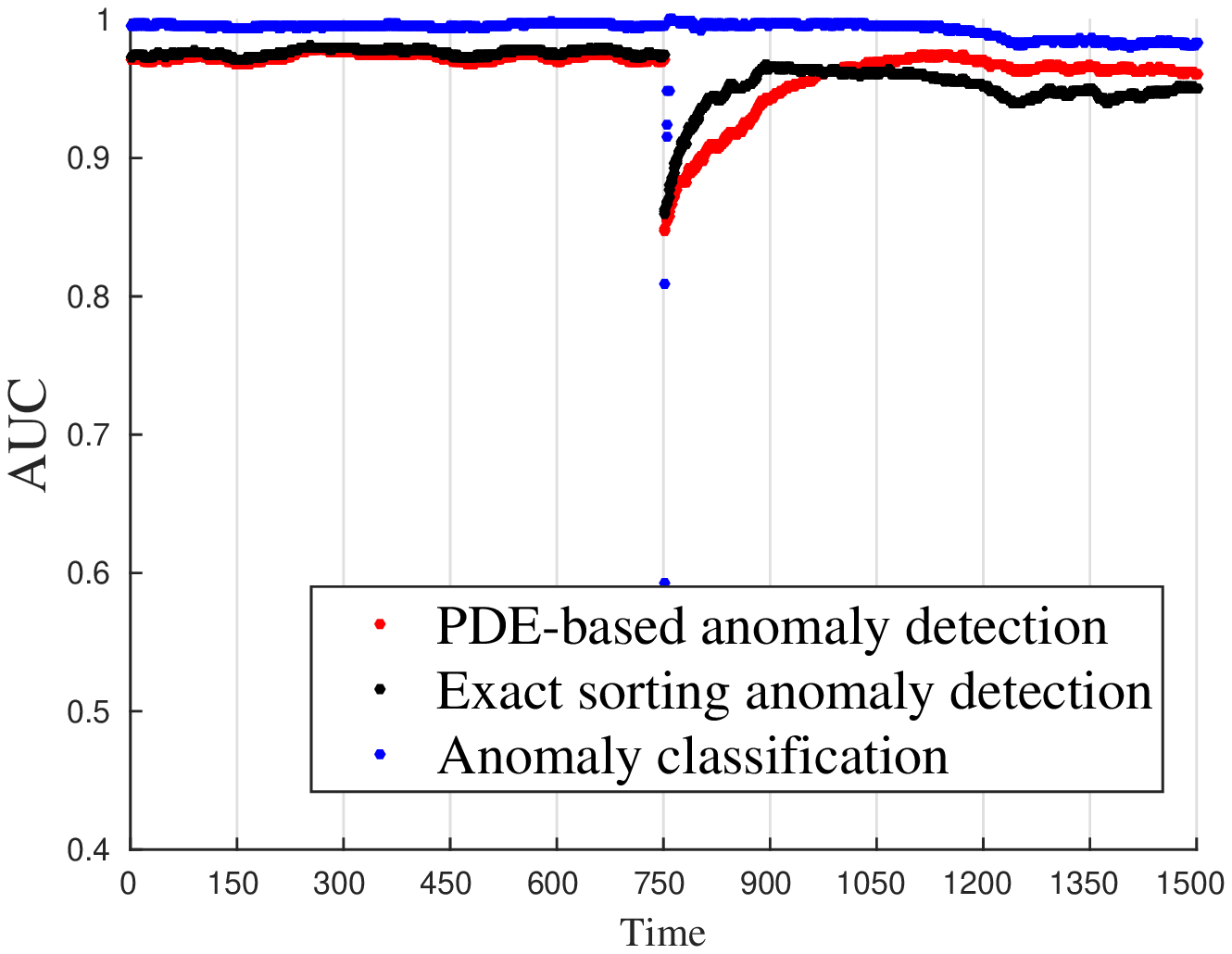}\label{fig:AUC_uniform}}
\subfigure[]{\includegraphics[clip=true,trim=15 0 20 20,width=0.5\textwidth]{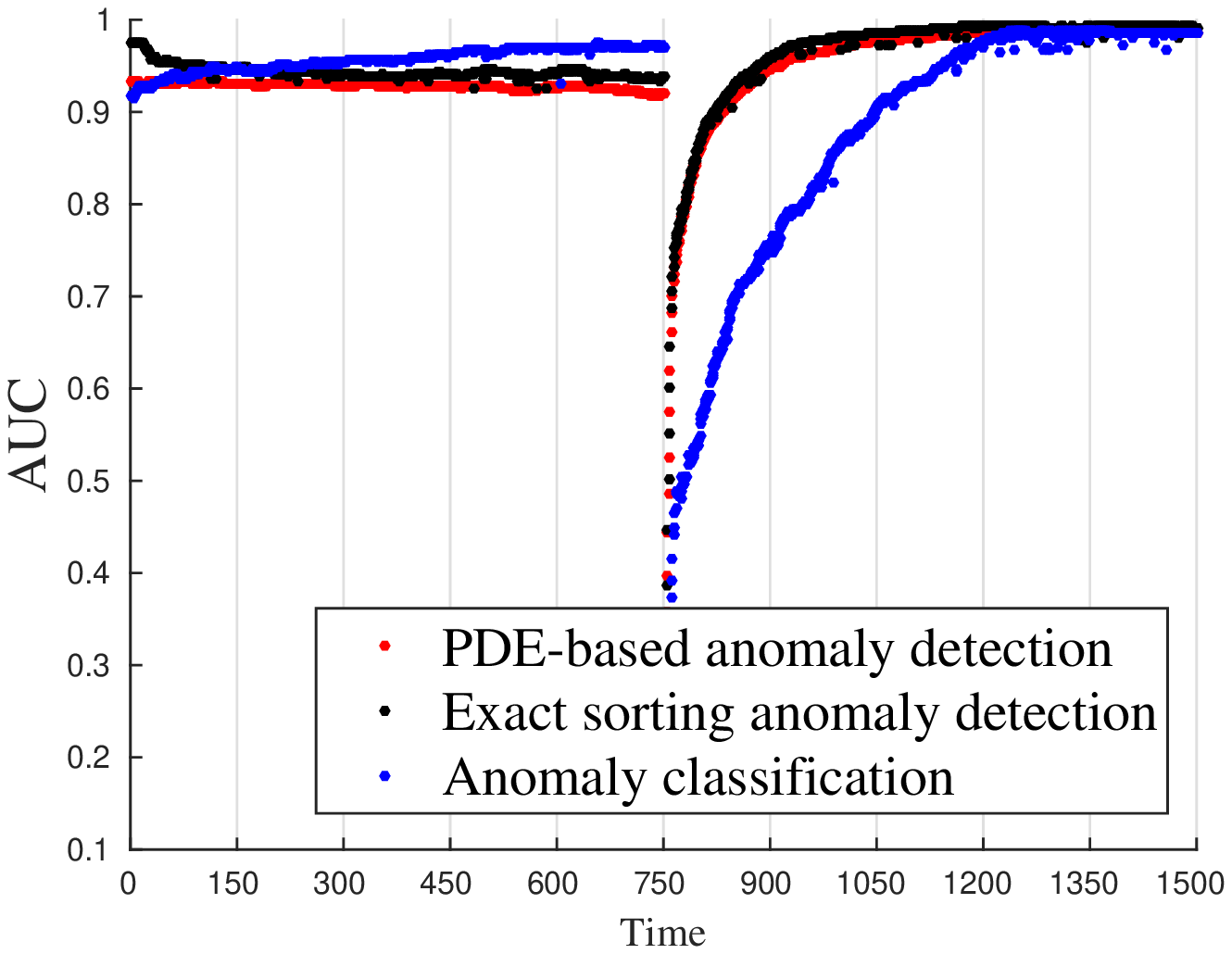}\label{fig:AUC_categorical}}
\subfigure[]{\includegraphics[clip=true,trim=15 0 20 20,width=0.5\textwidth]{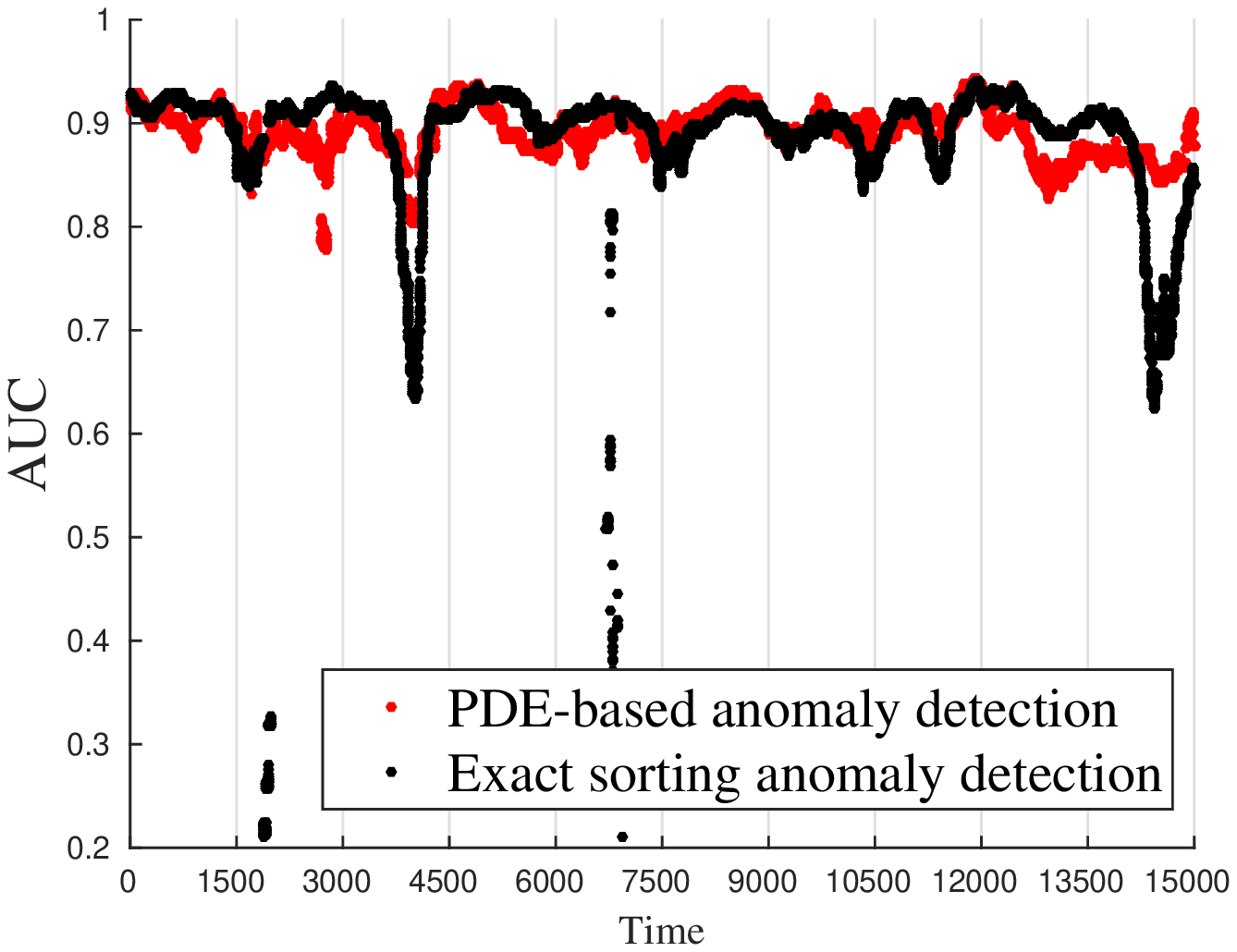}\label{fig:AUC_pedestrian}}
\caption{Results of simulated data stream with (a) uniformly distributed data and (b) categorical data. In (c) we show the AUCs for the pedestrian trajectories dataset.} 
\label{fig:AUC}
\end{figure}

\subsection{Uniformly distributed data}

The first experiment conducted with synthetic data took \emph{i.i.d.}~uniform samples on $[0,1]^2$ to be nominal, and uniform samples from the region $[0,1.1]^2\setminus[0,1]^2$ to be anomalous. Halfway through the stream the nominal region was changed to the box $[0,2]^2$, and the corresponding anomalous region was changed to $[0,2.2]^2\setminus[0,2]^2$. The two similarity criteria were simply taken to be the component-wise differences $|\Delta x_1|$ and $|\Delta x_2|$, respectively. The nearest neighbour parameters were chosen as $k_1=6,k_2=7$.  At each time step in the simulated stream there was a 0.05 probability of drawing from the anomalous region. 

Figure \ref{fig:AUC_uniform} shows the resulting AUCs at each time step. As expected, one can see a significant drop in the AUC of the anomaly detection at the mid-point when the trend is changed.  We observe a sharp recovery of the AUC of the anomaly detection once the training history $H_t$ contains a significant number of samples from the new distribution. This illustrates how the algorithm can quickly and efficiently learn a new trend in the data. Note that the AUC for the anomaly classification remains unchanged throughout the experiment because the classification of the anomalies in the new trend are the same with respect to the old trend.


\subsection{Categorical data}

\begin{figure}
\centering
\subfigure[]{\includegraphics[clip=true,trim=20 20 20 20,height=0.26\textheight]{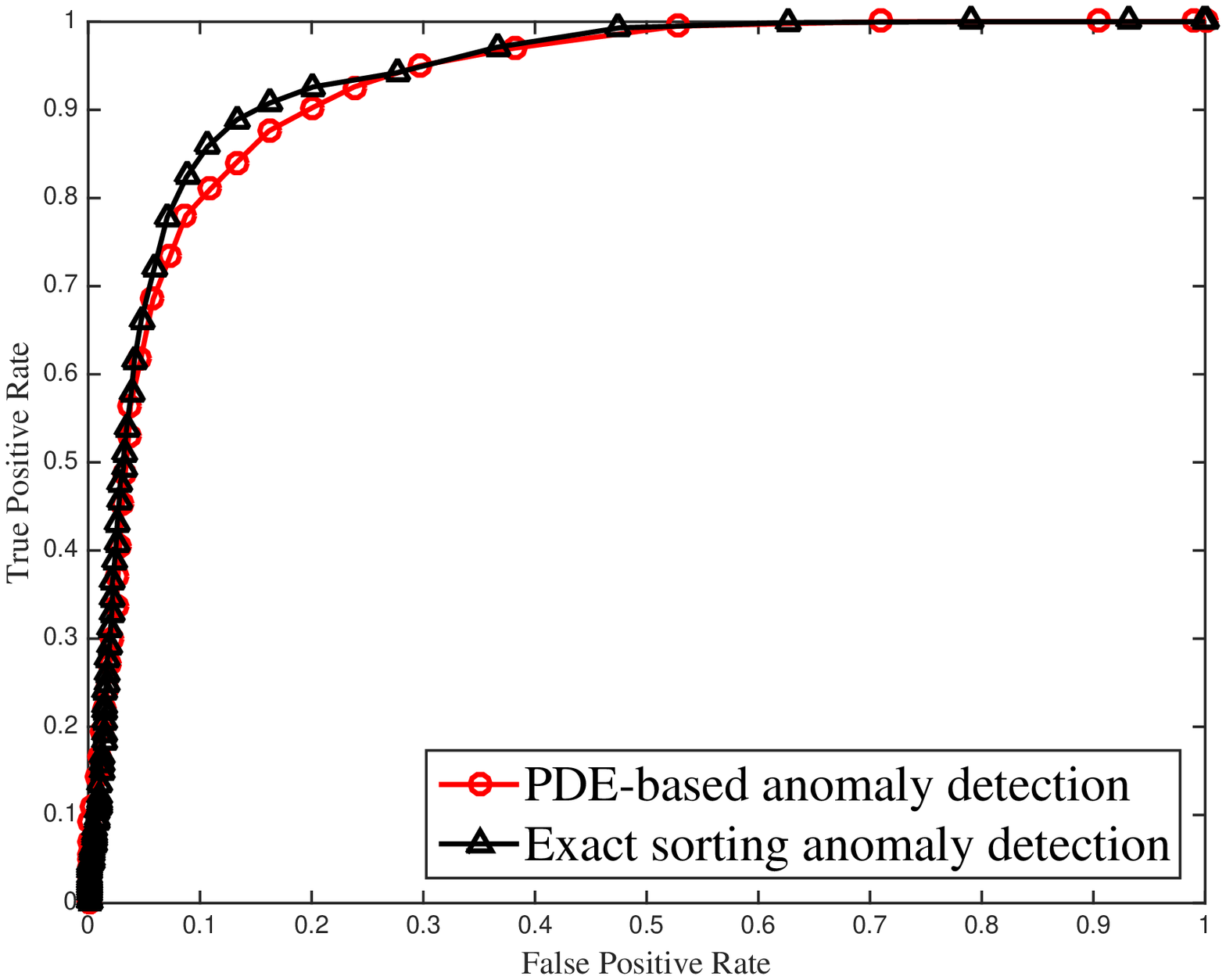}}
\subfigure[]{\includegraphics[clip=true,trim=20 20 20 20,height=0.26\textheight]{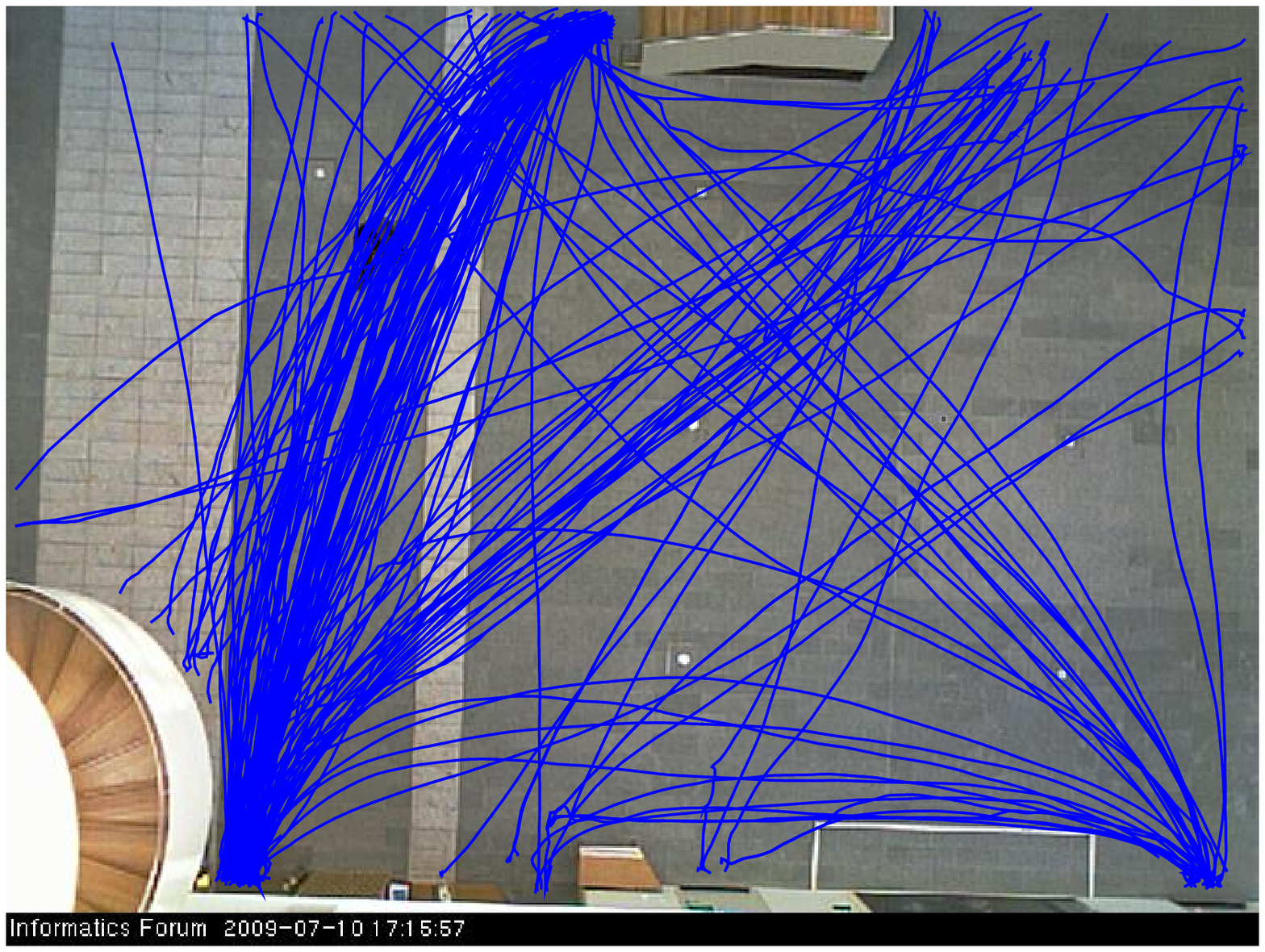}}
\subfigure[]{\includegraphics[clip=true,trim=20 20 20 20,height=0.26\textheight]{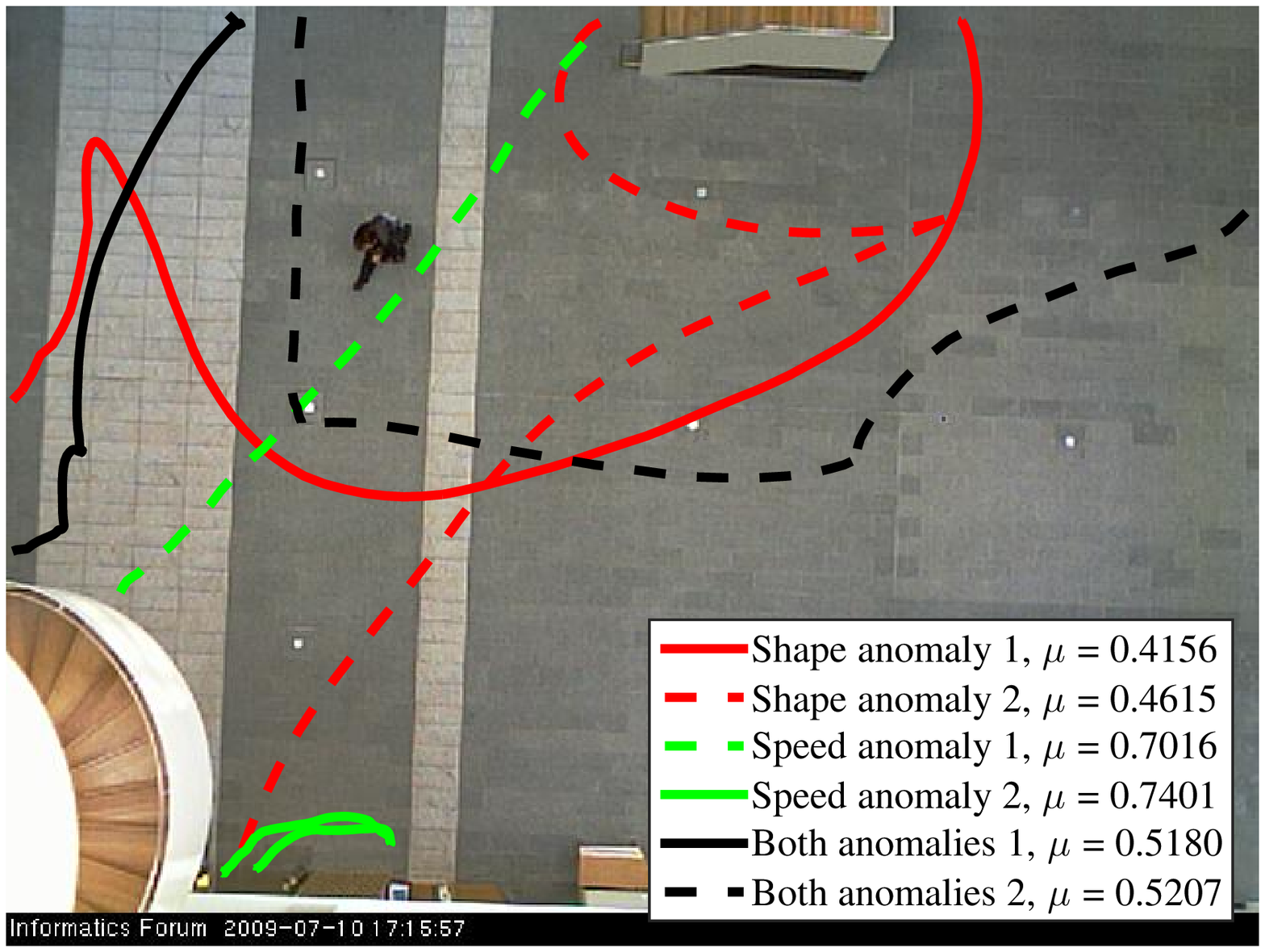}}
\caption{Pedestrian experiment in a stationary setting. (a) ROC curves of PDE-based and exact sorting anomaly detection, (b) trajectories classified as normal, and (c) some anomalous trajectories with their classifications.} 
\label{fig:pedestrian_traj}
\end{figure}
For the second experiment, we used the synthetic categorical data from \cite{hsiao2015}. Each sample consists of 2 groups of categorical data $A_1$ and $A_2$. Each group is comprised of 20 different attributes, where each attribute can assume a different number of values. The number of possible values for the $j^{\rm th}$ attribute of the $i^{\rm th}$ group, denoted $n_{i,j}$, is chosen uniformly at random between 6 and 10. Each attribute is then assigned a categorical distribution with parameters $p_1,\dots,p_{n_{i,j}}$ which are in turn drawn from a Dirichlet distribution with parameters $\alpha_{1},\dots,\alpha_{n_{i,j}}$. 

The nominal distribution is characterized by setting $\alpha_1=5$ and $\alpha_k=1$ for every $k\neq1$ for every attribute. This forces a bias towards attributes assuming the value one. For the anomalous distribution we set $\alpha_k=1$ for every $k$, so that no attribute has a bias towards assuming any particular value. Halfway through the stream the nominal distributions were changed so that for every attribute, the parameters of the categorical distribution were drawn from a Dirichlet distribution with parameters $\alpha_2=5$ and $\alpha_k=1$ for every $k\neq2$. This shifts the nominal bias towards the value two. The anomalous distribution was unchanged.  

To generate a nominal sample, we draw from the nominal distribution for each group. To generate an anomalous sample, we randomly choose a group with probability $0.5$ and draw from the anomalous distribution for that group, and nominal distribution for the other. At each time step in the stream there was a $0.05$ probability of drawing an anomalous sample.

The similarity between samples was computed between respective groups using the Inverse Occurrence Frequency (IOF) measure presented in \cite{boriah2008similarity}. The Goodall2 and Overlap metrics gave similar performance. The nearest neighbour parameters were chosen as $k_1=k_2=10$.  Figure \ref{fig:AUC_categorical} shows the resulting AUCs at each time step. Similar to the previous experiment we observe a drop in the AUC of the anomaly detection and a recovery thereafter. We also observe a similar drop in anomaly classification and the corresponding recovery. In contrast to the previous example, the new anomalies are anomalies in both criteria with respect to the old trend, so that the classification has no bias towards a specific criteria.

\subsection{Pedestrian trajectories}
Our third experiment consisted of data from a real pedestrian trajectory data set \cite{majecka2009statistical}, with over 100,000 trajectories. The first similarity criterion used to compare trajectories was their difference in shape, given by the $\ell_2$-distance between interpolated trajectories. The second was their difference in walking speed, given by the $\ell_2$-distance between the velocity histograms of each trajectory. 

As a preliminary experiment, we tested the anomaly detection and anomaly classification on 1666 trajectories from a single day. These trajectories were hand-labelled as normal or anomalous by \cite{hsiao2015}. In each experiment the training set consisted of 500 trajectories randomly drawn from a total of 1666 trajectories that day. The mean AUCs of the PDE-based and exact sorting based algorithms were $0.9274\pm0.0085$ and $0.9363\pm0.0072$, respectively and the ROC curves are shown in Figure \ref{fig:pedestrian_traj}(a). We observe very little difference between the exact sorting and the PDE-approximations.  We cannot present quantitative results for anomaly classification in this setting as there is no ground truth labeled data available. Along with some normal trajectories, we also plotted some anomalous trajectories with their respective classification scores in Figure \ref{fig:pedestrian_traj}(b,c). 

Finally, we applied the PDE-based streaming anomaly detection algorithm to a large portion of the pedestrian dataset, spanning over several days of data. Figure \ref{fig:AUC_pedestrian} shows the AUC as a function of artificial time for a simulated stream consisting of 15,000 trajectories with an initial training set of 400 randomly drawn trajectories. The small labeled portion of the dataset (approx.~1000 trajectories) was used to generate the ROC curves.

\section{Conclusion}
\label{sec:con}

In this paper, we showed how to use some recently discovered PDE continuum limits for nondominated sorting to perform anomaly detection and classification in real-time in a streaming setting. The classification is performed using new PDE continuum limits for ordering within the nondominated layers. We proved convergence rates for the continuum approximations and presented the results of numerical experiments with synthetic and real data that show our algorithm can adapt quickly and efficiently to a changing data stream. Although we focused in this paper on the anomaly detection problem, the ideas are not restricted to this context. Indeed, nondominated sorting is widely used in multiobjective optimization, and the ideas in this paper potentially apply to any such application, leaving many interesting problems for future work.

In particular, we outline below some directions for future work that we are currently investigating.

\begin{enumerate}
\item The arguments used in Section \ref{sec:new} to derive the new PDE continuum limits \eqref{eq:T} and \eqref{eq:TW} for sorting points within layers are not rigorous. We are currently investigating a rigorous proof of these conjectured continuum limits.  
\item The upwind finite difference schemes for the transport equations \eqref{eq:T} and \eqref{eq:TW} presented in Section \ref{sec:schemes} are provably convergent only when $u \in C^1$, which is not generally true, and only when we assume the exact values of $u_{x_1}$ and $u_{x_2}$ are used in the scheme. It would be interesting to prove convergence of these new upwind schemes without the assumption that $u \in C^1$, and under the more realistic condition that $u_{x_1}$ and $u_{x_2}$ are replaced by their finite difference approximations in the schemes for \eqref{eq:T} and \eqref{eq:TW}.
\item The PDEs for sorting points within fronts were presented in only $d=2$ dimensions here. It would be interesting to extend these results to higher dimensions. In $d\geq 3$ dimensions, there is no canonical linear ordering of the points within each front. Instead, we can consider nondominated sorting of the points within each front under a partial order that ``forgets'' about one of the coordinates $x_k$ (that is, we project to $\R^{d-1}$ by removing the $x_k$ coordinate and apply the usual nondominated sorting). This is akin to sorting the points within each front with respect to the $x_k$ direction. Thus, we have $d$ different ways to sort the points along each Pareto front, and similar arguments can be used to derive $d$ different continuum limit PDEs for sorting functions $v^1,\dots,v^k$ of the form
\[\prod_{i\neq k} \nabla v^k \cdot  \nabla^\perp_{k,i} u = u_{x_k}^{d-2}f \ \  \text{in } (0,1)^d,\]
with boundary condition $v^k=0$ on  $\{x_k=1\}$. Here, $\nabla^\perp_{k,i} u := u_{x_k}e_i - u_{x_i}e_k$, and $e_1,\dots,e_d$ are the standard basis vectors in $\R^d$. We are currently investigating applications of these continuum PDEs, as well as a rigorous proof in dimensions $d\geq 3$.
\end{enumerate}


\end{document}